\newif\iflongversion%
\longversiontrue%

\documentclass[sigconf,authorversion]{acmart}
\usepackage[utf8]{inputenc}
\usepackage[ruled,noline,noend,linesnumbered]{algorithm2e}
\usepackage{bm}
\usepackage{cleveref}
\crefname{algorithm}{Alg.}{Algs.}
\Crefname{algorithm}{Algorithm}{Algorithms}
\crefname{appendix}{App.}{App.}
\Crefname{appendix}{Appendix}{Appendices}
\crefname{corollary}{Corol.}{Corolls.}
\Crefname{corollary}{Corollary}{Corollaries}
\crefname{conjecture}{Conjecture}{Conjectures}
\Crefname{conjecture}{Conjecture}{Conjectures}
\crefformat{conjecture}{Conjecture~#2#1#3}
\crefmultiformat{conjecture}{Conjectures~#2#1#3}{\ and~#2#1#3}{, #2#1#3}{\ and~#2#1#3}
\crefname{definition}{Def.}{Defs.}
\Crefname{definition}{Definition}{Definition}
\crefname{figure}{Fig.}{Figs.}
\Crefname{figure}{Figure}{Figures}
\crefname{lemma}{Lemma}{Lemmas}
\Crefname{lemma}{Lemma}{Lemmas}
\crefname{proposition}{Prop.}{Props.}
\Crefname{proposition}{Proposition}{Propositions}
\Crefname{section}{Section}{Sections}
\crefname{section}{Sect.}{Sect.}
\crefname{subsection}{Sect.}{Sect.}
\Crefname{subsection}{Section}{Sections}
\crefname{subsubsection}{Sect.}{Sect.}
\Crefname{subsubsection}{Section}{Sections}
\crefname{table}{Table}{Tables}
\Crefname{table}{Table}{Tables}
\crefname{theorem}{Thm.}{Thms.}
\Crefname{theorem}{Theorem}{Theorems}
\usepackage{dsfont}
\usepackage{enumitem}
\usepackage{mathtools}

\usepackage{caption}
\usepackage{subcaption}

\DeclareMathOperator*{\E}{\mathbb{E}}
\DeclareMathOperator*{\EE}{\hat{\E}}

\newcommand{\algo}{\textup{\textsc{\algname}}}
\newcommand{\algoh}{\textup{\textsc{\algname-H}}}

\newcommand{\algotfp}{\textup{\textsc{TFP-R}}}
\newcommand{\algotfpamira}{\textup{\textsc{TFP-A}}}
\newcommand{\algname}{MCRapper}
\newcommand{\amira}{\textsc{Amira}}
\newcommand{\topkwy}{\textsc{TopKWY}}
\newcommand{\children}{\mathsf{children}}
\newcommand{\D}{\mathcal{D}}
\newcommand{\dbound}[2]{\Psi_{#2} ({#1})}
\newcommand{\ndbound}[2]{\widetilde{\Psi} ({#1})}

\newcommand{\desc}{\mathsf{d}}
\newcommand{\discr}[2]{\Delta_{#2}(#1)}

\newcommand{\F}{\mathcal{F}}

\newcommand{\frange}{c}
\newcommand{\freqset}{\mathcal{G}}
\newcommand{\iid}{i.i.d.}

\newcommand{\indp}{\vsigma^{+}_{j,i}}
\newcommand{\indm}{\vsigma^{-}_{j,i}}

\newcommand{\infreqset}{\mathcal{K}}
\newcommand{\Itms}{\mathcal{I}}
\newcommand{\lang}{\mathcal{L}}

\newcommand{\minimals}{\mathsf{minimals}}
\newcommand{\negborder}{\textsf{B}^-}

\newcommand{\pattern}{P}

\newcommand{\probdist}{\mu}
\newcommand{\R}{\mathbb{R}}
\newcommand{\rade}{\mathsf{R}}
\newcommand{\erade}{\hat{\rade}}
\newcommand{\trade}{\tilde{\rade}}

\newcommand{\era}{\erade\left(\F, \sample\right)}
\newcommand{\rmax}{z}
\newcommand{\sample}{\mathcal{S}}
\newcommand{\sd}{\mathsf{D}}
\newcommand{\supdev}{\sd(\F, \sample, \probdist)}
\newcommand{\tempset}{\mathcal{C}}
\newcommand{\tfp}{\mathsf{TFP}}

\newcommand{\vsigma}{{\bm{\sigma}}}
\newcommand{\wrt}{w.r.t.}
\newcommand{\X}{\mathcal{X}}

\newcommand{\papertitle}{\titlefirstpart\\\titlesecondpart}
\newcommand{\shortpapertitle}{\titlefirstpart\ \titlesecondpart}
\newcommand{\titlefirstpart}{\algname: Monte-Carlo Rademacher Averages}
\newcommand{\titlesecondpart}{for Poset Families and Approximate Pattern Mining}
\newcommand{\repolink}{\url{https://github.com/VandinLab/MCRapper}}

\copyrightyear{2020}
\acmYear{2020}
\setcopyright{acmlicensed}\acmConference[KDD '20]{Proceedings of the 26th ACM SIGKDD Conference on Knowledge Discovery and Data Mining}{August 23--27, 2020}{Virtual Event, CA, USA}
\acmBooktitle{Proceedings of the 26th ACM SIGKDD Conference on Knowledge Discovery and Data Mining (KDD '20), August 23--27, 2020, Virtual Event, CA, USA}
\acmPrice{15.00}
\acmDOI{10.1145/3394486.3403267}
\acmISBN{978-1-4503-7998-4/20/08}

\begin{document}
\iflongversion
\else
\fancyhead{} 
\fi

\title[\shortpapertitle]{\texorpdfstring{\papertitle}{\titlefirstpart\
\titlesecondpart}}

\author{Leonardo Pellegrina}
\affiliation{%
  \department{Dept.~of Information Engineering}
  \institution{Universit\`a di Padova}
  \streetaddress{Via G.~Gradenigo 6/B}
  \city{Padova}
  \postcode{IT-35131}
  \country{Italy}
}
\email{pellegri@dei.unipd.it}

\author{Cyrus Cousins}
\affiliation{%
  \department{Dept.~of Computer Science}
  \institution{Brown University}
  \streetaddress{115 Waterman St.}
  \city{Providence}
  \state{RI}
  \postcode{02912}
  \country{USA}
}
\email{ccousins@cs.brown.edu}

\author{Fabio Vandin}
\affiliation{%
  \department{Dept.~of Information Engineering}
  \institution{Universit\`a di Padova}
  \streetaddress{Via G.~Gradenigo 6/B}
  \city{Padova}
  \postcode{IT-35131}
  \country{Italy}
}
\email{fabio.vandin@unipd.it}

\author{Matteo Riondato}
\orcid{0000-0003-2523-4420}
\affiliation{%
  \department{Dept.~of Computer Science}
  \institution{Amherst College}
  \streetaddress{AC \#2232 Amherst College}
  \city{Amherst}
  \state{MA}
  \postcode{01002}
  \country{USA}
}
\email{mriondato@amherst.edu}

\begin{teaserfigure}
  \noindent\textit{``I'm an MC still as honest''} -- Eminem, \textit{Rap God}
  \Description{This is not a figure, it is just a quote before the abstract. The
    quote is: ``I'm an MC still as honest'', and it is taken from the song Rap
  God by Eminem.}
\end{teaserfigure}

\begin{abstract}
We present \algo, an algorithm for efficient computation of Monte-Carlo
Empirical Rademacher Averages (MCERA) for families of functions exhibiting poset (e.g.,
lattice) structure, such as those that arise in many pattern mining tasks. The
MCERA allows us to compute upper bounds to the maximum deviation of sample means
from their expectations, thus it can be used to find both
statistically-significant functions (i.e., patterns) when the available data is seen as
a sample from an unknown distribution, and approximations of collections of
high-expectation functions (e.g., frequent patterns) when the available data is
a small sample from a large dataset. This feature is a strong improvement over
previously proposed solutions that could only achieve one of the two. \algo\
uses upper bounds to the discrepancy of the functions to efficiently explore and
prune the search space, a technique borrowed from pattern mining itself. To show
the practical use of \algo, we employ it to develop an algorithm \algotfp\ for the task of
True Frequent Pattern (TFP) mining. \algotfp\ gives guarantees on the
probability of including any false positives (precision) and exhibits higher statistical
power (recall) than existing methods offering the same guarantees. We
evaluate \algo\ and \algotfp\ and show that they outperform the state-of-the-art for
their respective tasks.

\end{abstract}

\ccsdesc[500]{Information systems~Data mining}
\ccsdesc[500]{Mathematics of computing~Probabilistic algorithms}
\ccsdesc[500]{Theory of computation~Sketching and sampling}


\maketitle

\section{Introduction}\label{sec:intro}
Pattern mining is a key sub-area of knowledge discovery from data, with a large
number of variants (from itemsets mining~\citep{AgrawalIS93} to subgroup
discovery~\citep{Klosgen92}, to sequential patterns~\citep{AgrawalS95}, to
graphlets~\cite{AhmedNRD15}) tailored to applications ranging from market basket
analysis to spam detection to recommendation systems. Ingenuous algorithms 
have been proposed over the years, and pattern mining is both
extremely used in practice and a very vibrant area of research.

In this work we are interested in the analysis of \emph{samples} for pattern
mining. There are two meanings of ``sample'' in this context, but, as we now
argue, they are really two sides of the same coin, and our methods work for both
sides.

The first meaning is \emph{sample} as a \emph{small random sample of a large
dataset}: since mining patterns becomes more expensive as the dataset grows, it
is reasonable to mine only a small random sample that fits into the main memory
of the machine. Recently, this meaning of sample as
``sample-of-the-dataset'' has been used also to enable interactive data
exploration using progressive algorithms for pattern
mining~\citep{ServanSchreiberRZ18}. The patterns obtained from the sample are
an \emph{approximation} of the exact collection, due to the noise
introduced by the sampling process. To obtain desirable probabilistic guarantees
on the quality of the approximation, one must study the \emph{trade-off
between the size of the sample and the quality of the approximation}. Many works
have progressively obtained better characterizations of the trade-off using
advanced probabilistic
concepts~\citep{Toivonen96,ChakaravarthyPS09,RiondatoU14,RiondatoU15,RiondatoV18,ServanSchreiberRZ18}.
Recent methods~\citep{RiondatoU14,RiondatoU15,RiondatoV18,ServanSchreiberRZ18}
use VC-dimension, pseudodimension, and Rademacher
averages~\citep{BartlettM02,KoltchinskiiP00}, key concepts from statistical
learning theory~\citep{Vapnik98} (see also \cref{sec:relwork,sec:prelims:rade}),
because they allow to obtain uniform (i.e., simultaneous) probabilistic
guarantees on the deviations of all sample means (e.g., sample
frequencies, or other measure of interestingness, of all patterns) from their
expectations (the exact interestingness of the patterns in the dataset).

The second meaning is \emph{sample} as a \emph{sample from an unknown data
generating distribution}: the whole dataset is seen as a collection of samples
from an unknown distribution, and the goal of mining patterns from the available
dataset is to gain approximate information (or better, discover knowledge) about
the distribution. This area is known as \emph{statistically-sound pattern
discovery}~\citep{HamalainenW18}, and there are many different flavors of it,
from significant pattern mining~\citep{TeradaOHTS13} from transactional
datasets~\citep{PellegrinaRV19a,kirsch2012efficient},
sequences~\citep{tonon2019permutation}, or graphs~\citep{SugiyamaLLKB15}, to true
frequent itemset mining~\citep{riondato2014finding}, to, at least in part,
contrast pattern mining~\citep{BayP01}. Many works in this area also use
concepts from statistical learning theory such as empirical
VC-dimension~\citep{riondato2014finding} or Rademacher
averages~\citep{PellegrinaRV19a}, because, once again, these concepts allow to
get very sharp bounds on the maximum difference between the observed
interestingness on the sample and the unknown interestingness according to the
distribution.

The two meanings of ``sample'' are really two sides of the same
coin, because also in the first case the goal is to approximate an unknown
distribution from a sample, thus falling back into the second case. Despite this
similarity, previous contributions have been extremely point-of-view-specific
and pattern-specific. In part, these limitations are due to the techniques used
to study the trade-off between sample size and quality of the approximation
obtained from the sample. Our work instead proposes a \emph{unifying solution}
for mining approximate collections of patterns from samples, while giving
guarantees on the quality of the approximation: our proposed method can easily
be adapted to approximate collections of frequent itemsets, frequent sequences,
true frequent patterns, significant patterns, and many other tasks, even outside
of pattern mining.

At the core of our approach is the \emph{$n$-Samples Monte-Carlo (Empirical)
Rademacher Average ($n$-MCERA)}~\citep{BartlettM02} (see~\eqref{eq:mcera}),
which has the flexibility and the power needed to achieve our goals, as it gives
much sharper bounds to the deviation than other approaches. The
challenge in using the $n$-MCERA, like other quantities from statistical
learning theory, is how to compute it efficiently.

\paragraph{Contributions} We present \algo, an algorithm for the fast
computation of the $n$-MCERA of families of functions with a poset structure,
which often arise in pattern mining tasks (\cref{sec:prelims:patterns}).
\begin{itemize}[leftmargin=10pt]
  \item \algo\ is the first algorithm to compute
    the $n$-MCERA efficiently. It achieves this goal by using
    sharp upper bounds to the discrepancy of each function in the family
    (\cref{sec:bounds}) to quickly prune large parts of the function search
    space during the exploration necessary to compute the $n$-MCERA, in a
    branch-and-bound fashion. We also develop a novel sharper upper bound to the
    supremum deviation using the 1-MCERA (\cref{thm:supdev1mcera}). It holds
    for any family of functions, and is of independent interest.
  \item To showcase the practical strength of \algo, we develop \algotfp\
    (\cref{sec:appl}), a novel algorithm for the extraction of the True Frequent
    Patterns (TFP)~\citep{riondato2014finding}. \algotfp\ gives probabilistic
    guarantees on the quality of its output: with probability at least
    $1-\delta$ (over the choice of the sample and the randomness used in
    the algorithm), for user-supplied $\delta\in(0,1)$, the
    output is guaranteed to not contain any false positives. That is, \algotfp\
    controls the Family-Wise Error Rate (FWER) at level $\delta$ while
    achieving high statistical power, thanks to the use of the $n$-MCERA and
    of novel \emph{variance-aware} tail bounds (\cref{thm:supdevvar}). We also
    discuss other applications of \algo, to remark on its flexibility as a
    general-purpose algorithm.
  \item We conduct an extensive experimental evaluation of \algo\ and \algotfp\ on real
    datasets (\cref{sec:experiments}), and compare their performance with that of
    state-of-the-art algorithms for their respective
    tasks. \algo, thanks to the $n$-MCERA, computes much sharper
    (i.e, lower) upper bounds to the supremum deviation than algorithms using
    the looser Massart's lemma~\citep[Lemma
    26.8]{ShalevSBD14}. \algotfp\ extracts many more TFPs (i.e., has higher
    statistical power) than existing algorithms with the same guarantees.
\end{itemize}

\section{Related Work}\label{sec:relwork}
Our work applies to both the ``small-random-sample-from-large-dataset'' and the
``dataset-as-a-sample'' settings, so we now discuss the relationship of our work
to prior art in both settings. We do not study the important but different
task of output sampling in pattern mining~\citep{BoleyLPG11,DzyubaVLDR17}. We
focus on works that use concepts from statistical learning theory: these are the
most related to our work, and most often the state of the art in their areas.
More details are available in surveys~\citep{RiondatoU14,HamalainenW18}.

The idea of mining a small random sample of a large dataset to speed up the
pattern extraction step was proposed for the case of itemsets by
\citet{Toivonen96} shortly after the first algorithm for the task had been
introduced. The trade-off between the sample size and the quality of the
approximation obtained from the sample has been progressively better
characterized~\citep{ChakaravarthyPS09,RiondatoU14,RiondatoU15}, with large
improvements due to the use of concepts from statistical learning theory.
\citet{RiondatoU14} study the VC-dimension of the itemsets mining task, which
results in a worst-case \emph{dataset-dependent} but \emph{sample- and
distribution-agnostic characterization} of the trade-off. The major advantage of
using Rademacher averages~\citep{KoltchinskiiP00}, as we do in \algo\, is that
the characterization is now \emph{sample-and-distribution-dependent}, which
gives much better upper bounds to the maximum deviation of sample means from
their expectations. Rademacher averages were also used by \citet{RiondatoU15}, but
they used worst-case upper bounds (based on Massart's lemma~\citep[Lemma
26.2]{ShalevSBD14}) to the empirical Rademacher average of the task, resulting
in excessively large bounds. \algo\ instead computes the \emph{exact} $n$-MCERA
of the family of interest on the observed sample, without having to consider the
worst case. For other kinds of patterns, \citet{RiondatoV18} studied the
pseudodimension of subgroups, while \citet{ServanSchreiberRZ18} and \citet{SantoroTV20} considered the
(empirical) VC-dimension and Rademacher averages for sequential patterns.
\algo\ can be applied in all
these cases, and obtains better bounds because it uses the
\emph{sample-and-distribution-dependent} $n$-MCERA, rather than a worst case
dataset-dependent bound.

Significant pattern mining considers the dataset as a
sample from an unknown distribution. Many variants and algorithms are described
in the survey by \citet{HamalainenW18}. We discuss only the two most related to
our work. \citet{riondato2014finding} introduce the problem of finding the true
frequent itemsets, i.e., the itemsets that are frequent \wrt\ the unknown
distribution.  They propose a method based on empirical VC-dimension to compute
the frequency threshold to use to obtain a collection of true frequent patterns
with no false positives (see also \cref{sec:appl}). Our algorithm \algotfp\ uses
the $n$-MCERA, and as we show in \cref{sec:experiments}, it greatly outperforms
the state-of-the-art (a modified version of the algorithm by \citet{RiondatoU15}
for approximate frequent itemsets mining). \citet{PellegrinaRV19a} use
empirical Rademacher averages in their work for significant pattern mining. As
their work uses the bound by \citet{RiondatoU15}, the same comments about
the $n$-MCERA being a superior approach hold.

Our approach to bounding the supremum deviation by
computing the $n$-MCERA with efficient search space exploration techniques is
novel, not just in knowledge discovery, as the $n$-MCERA has received scant
attention.  \citet{RadaboundDSU} use it to control the generalization error in a
sequential and adaptive setting, but do not discuss efficient computation. We
believe that the lack of attention to the $n$-MCERA can be be explained by the
fact that there were no efficient algorithms for it, a gap now filled by  \algo.

\section{Preliminaries}\label{sec:prelims}
We now define the most important concepts and results that we use throughout this
work. Let $\F$ be a class of real valued functions from a domain $\X$ to the
interval $[a , b] \subset \R$. We use $\frange$ to denote $|b-a|$ and $\rmax$ to
denote $\max\{|a|, |b|\}$. In this work, we focus on a specific class of
families (see \cref{sec:prelims:patterns}). In pattern mining from transactional
datasets, $\X$ is the set of all possible transactions (or, e.g., sequences).
Let $\probdist$ be an \emph{unknown} probability distribution over $\X$ and the
\emph{sample} $\sample = \{s_1, \dotsc, s_m\}$ be a bag of $m$ \iid\ random
samples from $\X$ drawn according to $\probdist$. We discussed in
\cref{sec:intro} how in the pattern mining case, the sample may either be the
whole dataset (sampled according to an unknown distribution) or a random sample
of a large dataset (more details in \cref{sec:prelims:patterns}).  For each $f
\in \F$, we define its \emph{empirical sample average (or sample mean)
$\EE_{\sample}[f]$ on $\sample$} and  its \emph{expectation} $\E[f]$
respectively as
\[
  \EE_\sample[f] \doteq \frac{1}{m} \sum_{s_{i} \in \sample} f(s_{i}) \
  \text{and}\ \E[f] \doteq \E_{\probdist} \left[ \frac{1}{m} \sum_{s_{i} \in \sample}
  f(s_{i}) \right] \enspace.
\]
In the pattern mining case, the sample mean is the observed interestingness of a
pattern, e.g., its frequency (but other measures of interestingness can be
modeled as above, as discussed for subgroups by~\citet{RiondatoV18}), while the
expectation is the unknown exact interestingness that we are interested in
approximating, that is, either in the large datasets or \wrt\ the unknown data
generating distribution. We are interested in developing tight and
fast-to-compute upper bounds to the \emph{supremum deviation (SD) $\supdev$ of $\F$ on $\sample$} between the empirical sample average and the expectation
\emph{simultaneously} for all $f \in \F$, defined as
\begin{equation}\label{eq:supdev}
  \supdev = \sup_{f \in \F} \left|  \EE_{\sample}[f] - \E_{\probdist}[f] \right|
  \enspace.
\end{equation}
The supremum deviation allows to quantify how good the estimates obtained from
the samples are. Because $\probdist$ is unknown, it is not possible to compute
$\supdev$ exactly. We introduce concepts such as Monte-Carlo Rademacher Average
and results to compute such bounds in \cref{sec:prelims:rade}, but first we
elaborate on the specific class of families that we are interested in.

\subsection{Poset families and patterns}\label{sec:prelims:patterns}
A partially-ordered set, or \emph{poset} is a pair $(A, \preceq)$ where $A$ is a
set and $\preceq$ is a binary relation between elements of $A$ that is
reflexive, anti-symmetric, and transitive. Examples of posets include the $A =
\mathbb{N}$ and the obvious ``less-than-or-equal-to'' ($\le$) relation, and the
powerset of a set of elements and the ``subset-or-equal'' ($\subseteq$)
relation. For any element $y \in A$, we call an element $w \in A$, $w \neq y$ a
\emph{descendant} of $y$ (and call $y$ an \emph{ancestor} of $w$) if $y \preceq
w$. Additionally, if $y \preceq w$ and there is no $q \in A$, $q \neq y$, $q
\neq w$ such that $y \preceq q \preceq w$, then we say that $w$ is a
\emph{child} of $y$ and that $y$ is a \emph{parent} of $w$. For example, the set
$\{0, 2\}$ is a parent of the set $\{0, 2, 5\}$ and an ancestor of the set $\{0,
1, 2, 7\}$, when considering $A$ to be all possible subsets of integers and the
$\subseteq$ relation.

In this work we are interested in posets where $A$ is a family $\F$ of functions
as in \cref{sec:prelims:rade}, and the relation $\preceq$ is the following: for
any $f, g \in \F$
\begin{equation}\label{eq:preceq}
  f \preceq g\ \text{iff}
  \begin{cases}
    f(x) \ge g(x) & \text{for every}\ x \in \X\ \text{s.t.}\ f(x) \ge 0\\
    f(x) \le g(x) & \text{for every}\ x \in \X\ \text{s.t.}\ f(x) < 0
  \end{cases}
  \enspace.
\end{equation}
The very general but a bit complicated requirement often collapses to much
simpler ones as we discuss below. We aim for generality, as our goal is to
develop a unifying approach for many pattern mining tasks, for both
meanings of ``sample'', as discussed in \cref{sec:intro}. For now, consider for
example that requiring $|f(x)| \ge |g(x)|$ for every $x \in \X$ is a
specialization of the above more general requirement. We assume to have access
to a blackbox function $\children$ that, given any function $f \in \F$, returns
the list of children of $f$ according to $\preceq$, and to a blackbox function
$\minimals$ that, given $\F$, returns the \emph{minimal elements \wrt\
$\preceq$}, i.e., all the functions $f \in \F$ without any parents. We refer to
families that satisfy these conditions as \emph{poset families}, even if the
conditions are more about the relation $\preceq$ than about the family. We now
discuss how poset families arise in many pattern mining tasks.

In pattern mining, it is assumed to have a language $\lang$ containing the
patterns of interest. For example, in itemsets mining~\citep{AgrawalIS93},
$\lang$ is the set of all possible \emph{itemsets}, i.e., all non-empty subsets
of an alphabet $\Itms$ of \emph{items}, while in sequential pattern
mining~\citep{AgrawalS95}, $\lang$ is the set of sequences, and in subgroup
discovery~\citep{Klosgen92}, $\lang$ is set by the user as the set of patterns of
interest. In all these cases, for each pattern $\pattern \in \lang$, it is
possible to define a function $f_\pattern$ from the domain $\X$, which is the
set of all possible \emph{transactions}, i.e., elementary components of the
dataset or of the sample, to an appropriate co-domain $[a, b]$, such that
$f_\pattern(x)$ denotes the ``value'' of the pattern $\pattern$ on the
transaction $x$. For example, for itemsets mining, $\X$ is all the subsets of
$\Itms$ and $f_\pattern$ maps $\X$ to $\{0, 1\}$ so that $f_\pattern(x) = 1$ iff
$\pattern \subseteq x$ and $0$ otherwise. A consequence of this definition is
that $\EE_\sample[f_\pattern]$ is the \emph{frequency} of $\pattern$ in
$\sample$, i.e., the fraction of transaction of $\sample$ that contain the
pattern $\pattern$. A more complex (due to the nature of the patterns) but
similar definition would hold for sequential patterns. For the case of
\emph{high-utility itemset mining}~\citep{FournierVCWLTCN19}, the value of
$f_\pattern(x)$ would be the utility of $\pattern$ in the transaction $x$. The
family $\F$ is the set of the functions $f_\pattern$ for every pattern $\pattern
\in \lang$. Similar reasoning also applies to patterns on graphs, such as
graphlets~\citep{AhmedNRD15}.

Now that we have defined the set that we are interested in, let's comment on the
relation $\preceq$ that, together with the set, forms the poset. In the itemsets
case, for any two patterns $\pattern'$ and $\pattern'' \in \lang$, i.e., for any
two functions $f = f_{\pattern'}$ and $g = f_{\pattern''} \in \F$, it holds $f
\preceq g$ iff $\pattern' \subseteq \pattern''$. For sequences, the
\emph{subsequence relation} $\sqsubseteq$ defines $\preceq$ instead. In all
pattern mining tasks, the only minimal element of $\F$ \wrt\ $\preceq$ is the
empty itemset (or sequence) $\emptyset$. Our assumption to have access to the
blackboxes $\children$ and $\minimals$ is therefore very reasonable, because
computing these collections is extremely straightforward in all the pattern
mining cases we just mentioned and many others.

\subsection{Rademacher Averages}\label{sec:prelims:rade}
Here we present Rademacher averages~\citep{KoltchinskiiP00,BartlettM02} and
related results at the core of statistical learning theory~\citep{Vapnik98}. Our
presentation uses the most recent and sharper results, and we also introduce new
results (\cref{thm:supdevvar}, and later \cref{thm:supdev1mcera}) that may be of independent
interest. For an introduction to statistical learning theory and more details about Rademacher
averages, we refer the interested reader to the textbook by \citet{ShalevSBD14}.
In this section we consider a generic family $\F$, not necessarily a poset
family.

A key quantity to study the supremum deviation (SD) from~\eqref{eq:supdev} is the
\emph{empirical Rademacher average (ERA) $\era$ of $\F$ on
$\sample$}~\citep{KoltchinskiiP00,BartlettM02}, defined as follows.  Let
$\vsigma  =\langle \sigma_1, \dotsc, \sigma_m \rangle$ be a collection of $m$
\iid\ Rademacher random variables, i.e., each taking
value in $\{-1, 1\}$ with equal probability. The ERA of $\F$ on $\sample$ is the
quantity
\begin{equation}\label{eq:rade}
 \era \doteq \E_\vsigma \left[ \sup_{f \in \F } \frac{1}{m}
 \sum_{i=1}^m \sigma_i f(s_i) \right] \enspace.
\end{equation}
Computing the ERA $\era$ exactly is often intractable, due to
the \emph{expectation} over $2^{m}$ possible assignments for $\vsigma$, and the
need to compute a \emph{supremum} for each of these assignments, which precludes
many standard techniques for computing expectations. Bounds to
the SD are then obtained through efficiently-computable \emph{upper bounds}
to the ERA\@. Massart's lemma~\citep[Lemma 26.2]{ShalevSBD14} gives a
\emph{deterministic} upper bound to the ERA that is often very loose.
Monte-Carlo estimation allows to obtain an often sharper \emph{probabilistic}
upper bound to the ERA\@. For $n \ge 1$, let $\vsigma \in {\{-1, 1\}}^{n \times
  m}$ be a $n \times m$ matrix of \iid\ Rademacher random variables. The
  \emph{$n$-Samples Monte-Carlo Empirical Rademacher Average ($n$-MCERA)
    $\erade^{n}_{m}(\F, \sample, \vsigma)$ of $\F$ on $\sample$ using $\vsigma$}
    is~\citep{BartlettM02}
\begin{equation}\label{eq:mcera}
  \erade^{n}_{m}(\F, \sample, \vsigma) \doteq \frac{1}{n}\sum_{j=1}^{n} \sup_{f
  \in \F} \frac{1}{m} \sum_{s_{i} \in \sample} \vsigma_{j,i}
  f(s_{i}) \enspace.
\end{equation}
The $n$-MCERA allows to obtain probabilistic upper bounds to the SD as follows
(proof in  %
\iflongversion%
\cref{sec:proofs}%
\else
the appendix of the extended online version at \refmissing%
\fi
). In \cref{sec:1mcera} we show a novel improved bound for the special case
$n=1$ (\cref{thm:supdev1mcera}).

\begin{theorem}\label{thm:supdevmcera}
  Let $\eta \in (0, 1)$.  For ease of notation let
  \begin{equation}\label{eq:trade}
    \trade \doteq \erade^{n}_{m} \left( \F, \sample, \vsigma
    \right) + 2 \rmax \sqrt{\frac{\ln \frac{4}{\eta}}{2nm}} \enspace.
  \end{equation}
  With probability at least $1 - \eta$ over the choice of $\sample$ and
  $\vsigma$, it holds
  \begin{equation}\label{eq:supdevmcera}
    \supdev \leq 2\trade + \frac{\sqrt{\frange(4m \trade + \frange\ln
      \frac{4}{\eta}) \ln \frac{4}{\eta}}}{m} + \frac{\frange \ln
    \frac{4}{\eta}}{m} + \frange\sqrt{\frac{\ln \frac{4}{\eta}}{2m}} \enspace.
  \end{equation}
\end{theorem}

Sharper upper bounds to $\supdev$ can be obtained with the $n$-MCERA when more
information about $\F$ is available. The proof is in %
\iflongversion%
\cref{sec:proofs}%
\else
the appendix of the extended online version at \refmissing%
\fi%
. We use this result for a specific pattern mining task in \cref{sec:appl}.

\begin{theorem}\label{thm:supdevvar}
  Let $v$ be an upper bound to the variance of \emph{every} function in $\F$,
  and let $\eta \in (0,1)$. Define the following quantities
  \begin{align}
    \rho &\doteq \rade^n_m(\F, \sample, \vsigma) + 2 \rmax \sqrt{\frac{\ln
    \frac{4}{\eta}}{2 n m}},\label{eq:supdevvar-rho}\\
    r &\doteq \rho + \frac{1}{2m} \left( \sqrt{\frange \left( 4m \rho +
        \frange \ln \frac{4}{\eta} \right) \ln
      \frac{4}{\eta}} + \frange \ln
    \frac{4}{\eta} \right),\nonumber\\
        \varepsilon &\doteq 2 r + \sqrt{\frac{2 \ln \frac{4}{\eta}
    \left( v + 8 \frange r \right)}{m}}
    + \frac{2 \frange \ln \frac{4}{\eta}}{3m} \enspace. \label{eq:supdevvar-eps}
  \end{align}
  Then, with probability at least $1 - \eta$ over the choice of $\sample$ and
  $\vsigma$, it holds
  \[
    \supdev \le \varepsilon \enspace.
  \]
\end{theorem}

Due to the dependency on $\rmax$ in \cref{thm:supdevmcera,thm:supdevvar}, it is
often convenient to use $\erade^n_m(\F - \tfrac{\frange}{2}, \sample, \vsigma)$
in place of $\erade^n_m(\F, \sample, \vsigma)$ in the above theorems, where $\F
- \tfrac{\frange}{2}$ denotes the \emph{range-centralized} family of functions
obtained by shifting every function in $\F$ by $-\tfrac{\frange}{2}$. The
results still hold for $\supdev$ because the SD is invariant to shifting, but
the bounds  to the SD usually improve since the corresponding
$\rmax$ for the range-centralized family is smaller.

\section{\texorpdfstring{\algo}{\algname}}\label{sec:algo}
We now describe and analyze our algorithm \algo\ to efficiently compute the
$n$-MCERA (see~\eqref{eq:mcera}) for a family $\F$ with the binary relation
$\preceq$ defined in~\eqref{eq:preceq} and the blackbox functions $\children$
and $\minimals$ described in \cref{sec:prelims:patterns}.

\subsection{Discrepancy bounds}\label{sec:bounds}

 For $j \in \{ 1, \dotsc, n\}$, we
denote as the \emph{$j$-discrepancy $\discr{f}{j}$ of $f \in \F$ on $\sample$
\wrt\ $\vsigma$} the quantity
\[
  \discr{f}{j} \doteq \sum_{s_{i} \in \sample} \vsigma_{j,i} f(s_{i})
  \enspace.
\]
The $j$-discrepancy is not an \emph{anti-monotonic function}, in the sense that
it does not necessarily hold that $\discr{f}{j} \ge \discr{g}{j}$ for every
descendant $g$ of $f \in \F$. Clearly, it holds
\begin{equation}\label{eq:mceradiscr}
  \erade^{n}_{m}(\F, \sample, \vsigma) = \frac{1}{n m} \sum_{j=1}^{n}
  \sup_{f \in \F} \discr{f}{j} \enspace.
\end{equation}
A na\"{\i}ve computation of the $n$-MCERA would require enumerating \emph{all}
the functions in $\F$ and computing their $j$-discrepancies, $1 \le j \le n$, in
order to find each of the $n$ suprema. We now present novel easy-to-compute
\emph{upper bounds} $\ndbound{f}{j}$ and $\dbound{f}{j}$ to $\discr{f}{j}$ such
that $\ndbound{f}{j} \ge \discr{g}{j}$ and $\dbound{f}{j} \ge \discr{g}{j}$ for
every $g \in \desc(f)$, where $\desc(f)$ denote the set of the
\emph{descendants} of $f$ \wrt\ $\preceq$. This key property (which is a
generalization of \emph{anti-monotonicity} to posets) allows us to derive
efficient algorithms for computing the $n$-MCERA \emph{exactly} without
enumerating \emph{all} the functions in $\F$. Such algorithms take a
branch-and-bound approach using the upper bounds to $\discr{f}{j}$ to prune
large portions of the search space (see \cref{sec:algorithms}).

For every $j \in \{1, \dotsc, n\}$ and $i \in \{ 1, \dotsc, m\}$, let
\[
  \indp \doteq \mathds{1}(\vsigma_{j,i} = 1), \ \text{and}\ \indm \doteq
 \mathds{1}(\vsigma_{j,i} = -1)
\]
and for every $f \in F$ and $x \in \X$, define the functions
\[
  f^+(x) \doteq f(x) \mathds{1}(f(x) \ge 0), \ \text{and}\  f^-(x) \doteq  f(x)
  \mathds{1}(f(x) < 0) \enspace.
\]
It holds $f^+(x) \ge 0$ and $f^-(x) \le 0$ for every $f \in \F$ and $x \in \X$.
For every $j \in \{1, \dotsc, n\}$ and $f \in \F$, define
\begin{align}\label{eq:bounds}
  \ndbound{f}{j} &\doteq \sum_{s_{i} \in \sample} |f(s_i)| &\text{and} \nonumber\\
  \dbound{f}{j} &\doteq \sum_{s_{i} \in \sample} \indp f^+(s_i) - \sum_{s_i \in
  \sample} \indm f^-(s_i) &\enspace.
\end{align}
Computationally, these quantities are extremely straightforward to obtain. Both
$\ndbound{f}{j}$ and $\dbound{f}{j}$ are upper bounds to $\discr{f}{j}$ and to
$\discr{g}{j}$ for all $g \in \desc(f)$ (proof in %
\iflongversion%
\cref{sec:proofs}).
\else
the appendix of the extended online version at \refmissing).
\fi

\begin{theorem}\label{thm:discrbounds}
  For any $f \in \F$ and $j \in \{1, \dotsc, n\}$, it holds
  \[
    \max \left\lbrace  \discr{g}{j} \ :\ g \in \desc(f) \cup \{f\}
    \right\rbrace \leq \dbound{f}{j} \leq \ndbound{f}{j} \enspace.
  \]
\end{theorem}

The bounds we derived in this section are \emph{deterministic}. An interesting
direction for future research is how to obtain sharper \emph{probabilistic}
bounds.

\subsection{Algorithms}\label{sec:algorithms}

We now use the discrepancy bounds $\ndbound{\cdot}{\cdot}$ and
$\dbound{\cdot}{\cdot}$ from \cref{sec:bounds} in our algorithm \algo\ for
computing the exact $n$-MCERA\@. As the real problem is usually not to only
compute the $n$-MCERA but to actually compute an upper bound to the SD, our
description of \algo\ includes this final step, this also enables fair comparison with
existing algorithms that use \emph{deterministic} bounds to the ERA to compute
an upper bound to the SD (see also \cref{sec:experiments}).

\algo\ offers probabilistic guarantees on the quality of the bound it
computes (proof deferred to after the presentation).

\begin{theorem}\label{thm:algocorrectness}
  Let $\delta \in (0,1)$. With probability at least $1 - \delta$ over the choice
  of $\sample$ and of $\vsigma$, the value $\varepsilon$ returned by \algo\ is
  such that $\supdev \le \varepsilon$.
\end{theorem}

\begin{algorithm}[htb]
\SetNoFillComment%
\DontPrintSemicolon
  \KwIn{Poset family $\F$, sample $\sample$ of size $m$, $\delta \in
  (0,1)$, $n \geq 1$}
  \KwOut{Upper bound to $\supdev$ with probability $\geq 1 - \delta$.}
  \SetKw{Output}{output}
  \SetKwFunction{Children}{children}
  \SetKwFunction{Draw}{draw}
  \SetKwFunction{Process}{process}
  \SetKwFunction{GetNMCERA}{getNMCERA}
  \SetKwFunction{GetSupDevBound}{getSupDevBound}
  \SetKwFunction{Minimals}{minimals}
  \SetKwFunction{Pop}{pop}
  \SetKwFunction{Delete}{delete}
  \SetKwFunction{Push}{push}
  \SetKwProg{Fn}{Function}{:}{end}
  $\vsigma \gets \text{\Draw{$m$, $n$}}$\label{algline:draw-main}\;
  $\varepsilon \gets \text{\GetSupDevBound{$\F$, $\sample$, $\delta$,
  $\vsigma$}}$\label{algline:retval}\;
  \Return{$\varepsilon$}\;

  \Fn{\GetSupDevBound{$\F$, $\sample$, $\delta$, $\vsigma$}}{%
    $\trade \gets \text{\GetNMCERA{$\F$, $\sample$,
        $\vsigma$}} + 2 \rmax
        \sqrt{\frac{\ln(4/\delta)}{2nm}}$\label{algline:trade}\;
      \Return{r.h.s.\ of~\eqref{eq:supdevmcera} using
      $\eta=\delta$\label{algline:supdevmcera}}\;
  }

  \Fn{\GetNMCERA{$\F$, $\sample$, $\vsigma$}}{%
    $Q \gets$ empty priority queue\label{algline:queue}\;
    \lForEach{$j \in \{1, \dotsc, n\}$}{%
      $\nu_{j} \gets - zm$\label{eq:lbs}
    }
    $\mathcal{J} \gets$ empty dictionary from $\F$ to subsets of $\{1, \dotsc,
    n\}$\label{algline:cands}\;
    $H \gets \emptyset$\label{algline:pruned}\;
    \ForEach{$f \in \text{\Minimals{$\F$}}$}{\label{algline:minimals}
      $Q$.\Push{$f$}\;
      $\mathcal{J}[f] \gets \{1, \dotsc, n\}$\label{algline:candsminimals}\;
    }
    \While{$Q$ is not empty}{\label{algline:queueloop}
      $f \gets Q$.\Pop{}\label{algline:pop}\;
      $Y \gets \emptyset$\label{algline:nextCands}\;
      \ForEach{$j \in \mathcal{J}[f]$ s.t.~$\ndbound{f}{} \ge
        \nu_j$}{\label{algline:candloop}
        \If{$\dbound{f}{j} \ge \nu_j$}{\label{algline:dbound}
          $\nu_{j} \gets \max \{ \nu_{j}, \discr{f}{j} \}$\label{algline:discr}\;
           $Y \gets Y \cup \{j\}$\label{algline:updateNextCands}
        }
      }
      \ForEach{$g \in \text{\Children{$f$}} \setminus
        H$}{\label{algline:childrenloop}
        \lIf*{$g \in \mathcal{J}$}{%
          $N \gets \mathcal{J}[g] \cap Y$
        }
        \lElse{%
          $N \gets Y$\label{algline:childCands}
        }
        \If{$N =\emptyset$}{%
          $H \gets H \cup \{g \}$\label{algline:doPrune}\;
          \lIf{$g \in \mathcal{J}$}{%
            $Q$.\Delete{$g$}\label{algline:delete}
          }
        }
        \Else{%
          \lIf{$g \not\in \mathcal{J}$}{%
            $Q$.\Push{$g$}\label{algline:push}
          }
          $\mathcal{J}[g] \gets N$\label{algline:updateCands}
        }
      }
    }
    \Return{$\frac{1}{n m} \sum_{j=1}^{n} \nu_{j}$}\label{algline:mcera}\;
  }
  \caption{\algo}\label{algo:main}
\end{algorithm}

The pseudocode of \algo\ is presented in \cref{algo:main}. The division in
functions is useful for reusing parts of the algorithm in later sections (e.g.,
\cref{algo:tfp}). After having sampled the $n \times m$ matrix of \iid\
Rademacher random variables (line~\ref{algline:draw-main}), the algorithm calls
the function $\GetSupDevBound$ with appropriate parameters, which in turn calls
the function \GetNMCERA, the real heart of the algorithm. This function computes
the $n$-MCERA $\erade^n_m(\F, \sample, \vsigma)$ by exploring and pruning the
search space (i.e., $\F$) in according to the order of the elements in the
priority queue $Q$ (line~\ref{algline:queue}). One possibility is to explore the
space in Breadth-First-Search order (so $Q$ is just a FIFO queue), while another
is to use the upper bound $\ndbound{f}{}$ as the priority, with the top element
in the queue being the one with maximum priority among those in the queue. Other
orders are possible, but we assume that the order is such that all parents of a
function are explored before the function, which is reasonable to ensure maximum
pruning, and is satisfied by the two mentioned orders. We assume that the
priority queue also has a method \Delete{$e$} to delete an element $e$ in the
queue. This requirement could be avoided with some additional book-keeping, but
it simplifies the presentation of the algorithm.

The algorithm keeps in the quantities $\nu_j$, $j \in \{1, \dotsc, n\}$, the
currently best available lower bound to the quantity $\sup_{f \in \F}
\discr{f}{j}$ (see~\eqref{eq:mceradiscr}), which initially are all $-zm$ (the
lowest possible value of a discrepancy). \algo\ also maintains a dictionary
$\mathcal{J}$ (line~\ref{algline:cands}), initially empty, whose keys will be
elements of $\F$ and the values are subsets of $\{1, \dotsc, n\}$. The value
associated to a key $f$ in the dictionary is a superset of the set of values $j
\in \{1, \dotsc, n\}$ for which $\ndbound{f}{} \ge \nu_i$, i.e., for which $f$ or
one of its descendants \emph{may} be the function attaining the supremum
$j$-discrepancy among all the functions in $\F$ (see~\eqref{eq:mceradiscr}). A
function and all its descendants are pruned when this set is the empty set. The
set of keys of the dictionary $\mathcal{J}$ is, at all times, the set of all and
only the functions in $\F$ that have ever been added to $Q$. The last data
structure is the set $H$ (line~\ref{algline:pruned}), initially empty, which
will contain pruned elements of $\F$, in order to avoid visiting either them or
their descendants.

\algo\ populates $Q$ and $\mathcal{J}$ by inserting in them the minimal elements
of $\F$ \wrt\ $\preceq$ (line~\ref{algline:minimals}), using the set $\{1,
\dotsc, n\}$ as the value for these keys in the dictionary. It then enters a
loop that keeps iterating as long as there are elements in $Q$
(line~\ref{algline:queueloop}). The top element $f$ of $Q$ is extracted at the
beginning of each iteration (line~\ref{algline:pop}). A set $Y$, initially
empty, is created to maintain a superset of the set of values $j \in \{1,
\dotsc, n\}$ for which a child of $f$ \emph{may} be the function attaining the
supremum $j$-discrepancy among all the functions in $\F$
(see~\eqref{eq:mceradiscr}). The algorithm then iterates over the elements $j
\in \mathcal{J}[f]$ s.t.~$\ndbound{f}{}$ is greater than $\nu_j$
(line~\ref{algline:candloop}). The elements for which $\ndbound{f}{} < v_j$ can
be ignored because $f$ and its descendants can not attain the supremum of the
$j$-discrepancy in this case, thanks to \cref{thm:discrbounds}. Computing
$\ndbound{f}{}$ is straightforward and can be done even faster if one keeps a
frequent-pattern tree or a similar data structure to avoid having to scan
$\sample$ all the times, but we do not discuss this case for ease of
presentation. For the values $j$ that satisfy the condition on
line~\ref{algline:candloop}, the algorithm computes $\discr{f}{j}$ and updates
$\nu_j$ to this value if larger than the current value of $\nu_j$
(line~\ref{algline:discr}), to maintain the invariant that $\nu_j$ stores the
highest value of $j$-discrepancy seen so far (this invariant, together with the
one maintained by the pruning strategy, is at the basis of the correctness of
\algo). Finally, $j$ is added to the set $Y$
(line~\ref{algline:updateNextCands}), as it may still be the case that a
descendant of $f$ has $j$-discrepancy higher than $\nu_j$. The
algorithm then iterates over the children of $f$ that have not been pruned,
i.e., those not in $H$ (line~\ref{algline:childrenloop}). If the child $g$ is
such that there is a key $g$ in $\mathcal{J}$ (because before $f$ we visited
another parent of $g$), then let $N$ be $\mathcal{J}[g] \cap Y$, otherwise, let
$N$ be $Y$. The set $N$ is a superset of the indices $j$ s.t.~$g$ may attain the
supremum $j$-discrepancy. Indeed for a value $j$ to have this property, it is
\emph{necessary} that $\dbound{f}{j} \ge \nu_j$ \emph{for every parent $f$} of
$j$ (where the value of $\nu_j$ in this expression is the one that $\nu_j$ had
when $f$ was visited). If $N = \emptyset$, then $g$ and all its descendants can
be pruned, which is achieved by adding $g$ to $H$ (line~\ref{algline:doPrune})
and removing $g$ from $Q$ if it is a key $\mathcal{J}$
(line~\ref{algline:delete}). When $N \neq \emptyset$, first $g$ is added to $Q$
(with the appropriate priority depending on the ordering of $Q$) if it did not
belong to $\mathcal{J}$ yet (line~\ref{algline:push}), and then $\mathcal{J}[g]$
is set to $N$ (line~\ref{algline:updateCands}). This operation completes the
current loop iteration starting at line~\ref{algline:queueloop}.

Once $Q$ is empty, the loop ends and the function \GetNMCERA{} returns the sum
of the values $\nu_j$ divided by $n\cdot m$. The returned value is summed to an
appropriate term to obtain $\trade$ (line~\ref{algline:trade}), which is used to
compute the return value of the function \GetSupDevBound{}
using~\eqref{eq:supdevmcera} with $\eta = \delta$
(line~\ref{algline:supdevmcera}). This value $\varepsilon$ is returned in output
by \algo\ when it terminates (line~\ref{algline:retval}).

The following result is at the core of the correctness of \algo\ (proof in
\cref{sec:proofs}.)

\begin{lemma}\label{lem:getNMCERA}
  \GetNMCERA{$\F$, $\sample$, $\vsigma$} returns the value $\erade^n_m(\F,
  \sample, \vsigma)$.
\end{lemma}

The proof of \cref{thm:algocorrectness} is then just an application of
\cref{lem:getNMCERA} and \cref{thm:supdevmcera} (with $\eta = \delta$), as the
value $\varepsilon$ returned by \algo\ is computed according
to~\eqref{eq:supdevmcera}.

\subsubsection{Limiting the exploration of the search
space}\label{sec:hybridalgorithm}
Despite the very efficient pruning strategy made possible by the upper bounds to
the $j$-discrepancy, \algo\  may still need to explore a large fraction of the
search space, with negative impact on the running time. We now present a
``hybrid'' approach that limits this exploration, while still ensuring the
guarantees from \cref{thm:algocorrectness}.

Let $\beta$ be any positive value and define
\[
  \freqset(\sample , \beta) \doteq \left\lbrace f \in \F : \frac{1}{m}
  \sum_{i=1}^{m} {(f(s_{i}))}^{2} \geq \beta \right\rbrace,
\]
and $\infreqset(\sample , \beta) = \F \setminus \freqset(\sample, \beta)$. In
the case of itemsets mining, $\freqset(\sample , \beta)$ would be the set of
frequent itemsets \wrt\ $\beta\in[0,1]$.

The following result is a consequence of Hoeffding's inequality and a union
bound over $ n \cdot |\infreqset(\sample, \beta)|$ events.

\begin{lemma}\label{lem:hybrid}
  Let $\eta \in (0,1)$. Then, with probability at least $1 - \eta$ over the
  choice of $\vsigma$, it holds that \emph{simultaneously for all $j \in \{ 1,
  \dotsc, n\}$},
  \begin{equation}\label{eq:hybrid}
    \erade^1_m (\infreqset(\sample, \beta), \sample, \vsigma_j) \le
    \sqrt{\frac{2 \beta \log\left( \frac{n \left| \infreqset(\sample , \beta)
          \right|}{\eta} \right)}{m}} \enspace.
  \end{equation}
\end{lemma}

The following is an immediate consequence of the above and the definition of
$n$-MCERA.\@

\begin{theorem}\label{thm:resnmcera}
  Let $\eta \in (0,1)$. Then with probability $\geq 1 - \eta$ over the choice of
  $\vsigma$, it holds
  \begin{align*}
  &\erade^{n}_{m}(\F, \sample, \vsigma)
  = \frac{1}{n} \sum_{j=1}^{n} \max\left\lbrace \erade^{1}_{m}(\freqset(\sample , \beta), \sample, \vsigma_{j}) , \erade^{1}_{m}(\infreqset(\sample , \beta), \sample, \vsigma_{j}) \right\rbrace \\
  & \leq \frac{1}{n} \sum_{j=1}^{n} \max\left\lbrace \erade^{1}_{m}(\freqset(\sample , \beta), \sample, \vsigma_{j}) ,
    \sqrt{\frac{2 \beta \log\left( \frac{n \left| \infreqset(\sample , \beta) \right|}{\eta} \right)}{m}}
  \right\rbrace \enspace.
\end{align*}
\end{theorem}

The result of \cref{thm:resnmcera} is especially useful in situations when it is
possible to compute efficiently reasonable upper bounds on the cardinality of
$\infreqset(\sample , \beta)$, possibly using information from $\sample$ (but
not $\vsigma$). For the case of pattern mining, these bounds are often easy to
obtain: e.g., in the case of itemsets, it holds $|\infreqset(\sample , \beta)|
\le \sum_{s_i \in \sample} 2^{|s_i|}$, where $|s_i|$ is the number of items in
the transaction $s_i$. Much better bounds are possible, and in many other cases,
but we cannot discuss them here due to space limitations.

Combining the above with \algo\ may lead to a significant speed-up thanks
to the fact that \algo\ would be exploring only (a subset of) $\freqset(\sample
, \beta)$ instead of (a subset of) the entire search space $\F$, at the cost of
computing an \emph{upper bound} to $\erade^{n}_{m}(\F, \sample, \vsigma_{j})$,
rather than its exact value. We study this trade-off, which is governed by
the choice of $\beta$, experimentally in \cref{sec:hybridresults}. The
correctness follows from
\cref{thm:algocorrectness,thm:resnmcera,thm:supdevmcera}, and an application of
the union bound.

We now describe this variant \algoh{} of \algo, presented in \cref{algo:mcerah}.
\algoh\ accepts in input the same parameters of \algo, but also the parameters
$\beta$ and $\gamma < \delta$, which controls
the confidence of the probabilistic bound from \cref{thm:resnmcera}. After
having drawn $\vsigma$, \algoh\ computes the upper bound to $|\infreqset(\sample,\beta)|$ (line~\ref{algline:upperboundinfreq}), and calls the function
\GetNMCERA{$\freqset(\sample , \beta)$, $\sample$, $\vsigma$}
(line~\ref{algline:mceracall}), slightly modified \wrt\ the one on
line~\ref{algline:mcera} of \cref{algo:main} so it returns the set of $n$ values
$\{ \nu_{1} , \dotsc , \nu_{n} \}$ instead of their average. Then, it computes
$\trade$ using the r.h.s.\ of~\eqref{eq:hybrid} and returns the bound to the SD
obtained from the r.h.s.~of~\eqref{eq:supdevmcera} with $\eta = \delta -
\gamma$.

\begin{algorithm}[htb]
\SetNoFillComment%
\DontPrintSemicolon
  \KwIn{Poset family $\F$, sample $\sample$ of size $m$, $\delta \in
  (0,1)$, $\beta \in [0,\rmax^{2}]$, $\gamma \in (0,\delta)$}
  \KwOut{Upper bound to $\supdev$ with prob. $\geq 1 - \delta$.}
  \SetKw{Output}{output}
  \SetKwFunction{Children}{children}
  \SetKwFunction{Draw}{draw}
  \SetKwFunction{Process}{process}
  \SetKwFunction{GetNMCERA}{getNMCERA}
  \SetKwFunction{GetSupDevBound}{getSupDevBound}
  \SetKwFunction{Minimals}{minimals}
  \SetKwFunction{Pop}{pop}
  \SetKwFunction{Delete}{delete}
  \SetKwFunction{Push}{push}
  \SetKwProg{Fn}{Function}{:}{end}
  $\vsigma \gets \text{\Draw{$m$, $n$}}$\;
  $\{ v_{1} , \dotsc , v_{n} \} \gets \text{\GetNMCERA{$\freqset(
      \sample , \beta)$, $\sample$, $\vsigma$}}$\;\label{algline:mceracall}
    $\omega \gets $ upper bound to $\left| \infreqset(\sample , \beta) \right|
    $\label{algline:upperboundinfreq}\;
  $\trade \gets \frac{1}{n} \sum_{j=1}^{n} \max\left\lbrace \frac{{v}_{j}}{m},
  \sqrt{\frac{2 \beta \log\left( \frac{n \omega}{\gamma} \right)}{m}}
  \right\rbrace + 2 \rmax \sqrt{ \frac{\ln \left( \frac{4}{\delta-\gamma}
  \right)}{2nm}}$\;
  \Return{r.h.s.\ of~\eqref{eq:supdevmcera} using $\eta=\delta-\gamma$}\;
  \caption{\algoh}\label{algo:mcerah}
\end{algorithm}

It is not necessary to choose $\beta$ a-priori, as long as it is chosen without
using any information that depends on $\vsigma$.  In situations where deciding
$\beta$ a-priori is not simple, one may define instead, for a given value of
$k$ set by the user, the quantity $\beta_{k}$ defined as
\[
\beta_{k} \doteq \min \left\lbrace \beta : \left| \freqset(\sample , \beta) \right| \leq k \right\rbrace.
\]
When the queue $Q$ (line~\ref{algline:queue} of \cref{algo:main}) is sorted by
decreasing value of $\sum_{i=1}^n {(f(s_i))}^2$, the value $k$ is the maximum
number of nodes the branch-and-bound search in \GetNMCERA{} may enumerate.
We are investigating more refined bounds than \cref{thm:resnmcera}.

\subsection{Improved bounds for \texorpdfstring{$n=1$}{n=1}}\label{sec:1mcera}

For the special case of $n=1$, it is possible to derive a better bound to the
SD than the one presented in \cref{thm:supdevmcera}. This result is new and of
independent interest because it holds for \emph{any} family $\F$. The proof is
in %
\iflongversion%
\cref{sec:proofs}.
\else
the appendix of the extended online version at \refmissing.
\fi

\begin{theorem}\label{thm:supdev1mcera}
  Let $\eta \in (0, 1)$. With probability at least $1 - \eta$ over the choice of
  $\sample$ and $\sigma$, it holds that
  \begin{equation}\label{eq:supdev1mcera}
    \supdev \le 2\erade^{1}_{m} \left( \F - \frac{\frange}{2}, \sample, \vsigma
    \right) + 3 \frange \sqrt{\frac{\ln \frac{2}{\eta}}{2m}} \enspace.
  \end{equation}
\end{theorem}

The advantage of~\eqref{eq:supdev1mcera} over~\eqref{eq:supdevmcera} (with
$n=1$) is in the smaller ``tail bounds'' terms that arise thanks to a single
application of a probabilistic tail bound, rather than three such applications.
To use this result in \algo, line~\ref{algline:retval} must be replaced with
\[
  \varepsilon \gets\ \text{\GetNMCERA{$\F$, $\sample$, $\vsigma$}} + 3 \frange
  \sqrt{\frac{\ln \frac{2}{\delta}}{2m}};
\]
so the upper bound to the SD is computed according to~\eqref{eq:supdev1mcera}.
The same guarantees as in \cref{thm:algocorrectness} hold for this modified
algorithm.


\begin{figure*}[htb]
\iflongversion
\else
\vspace{-10pt}
\fi
\centering
\begin{subfigure}{.75\textwidth}
  \centering
  \includegraphics[width=\textwidth]{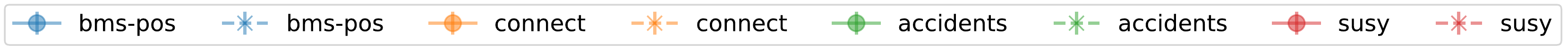}
\end{subfigure}%
\centering
\\
\begin{subfigure}{.3\textwidth}
  \centering
  \includegraphics[width=\textwidth]{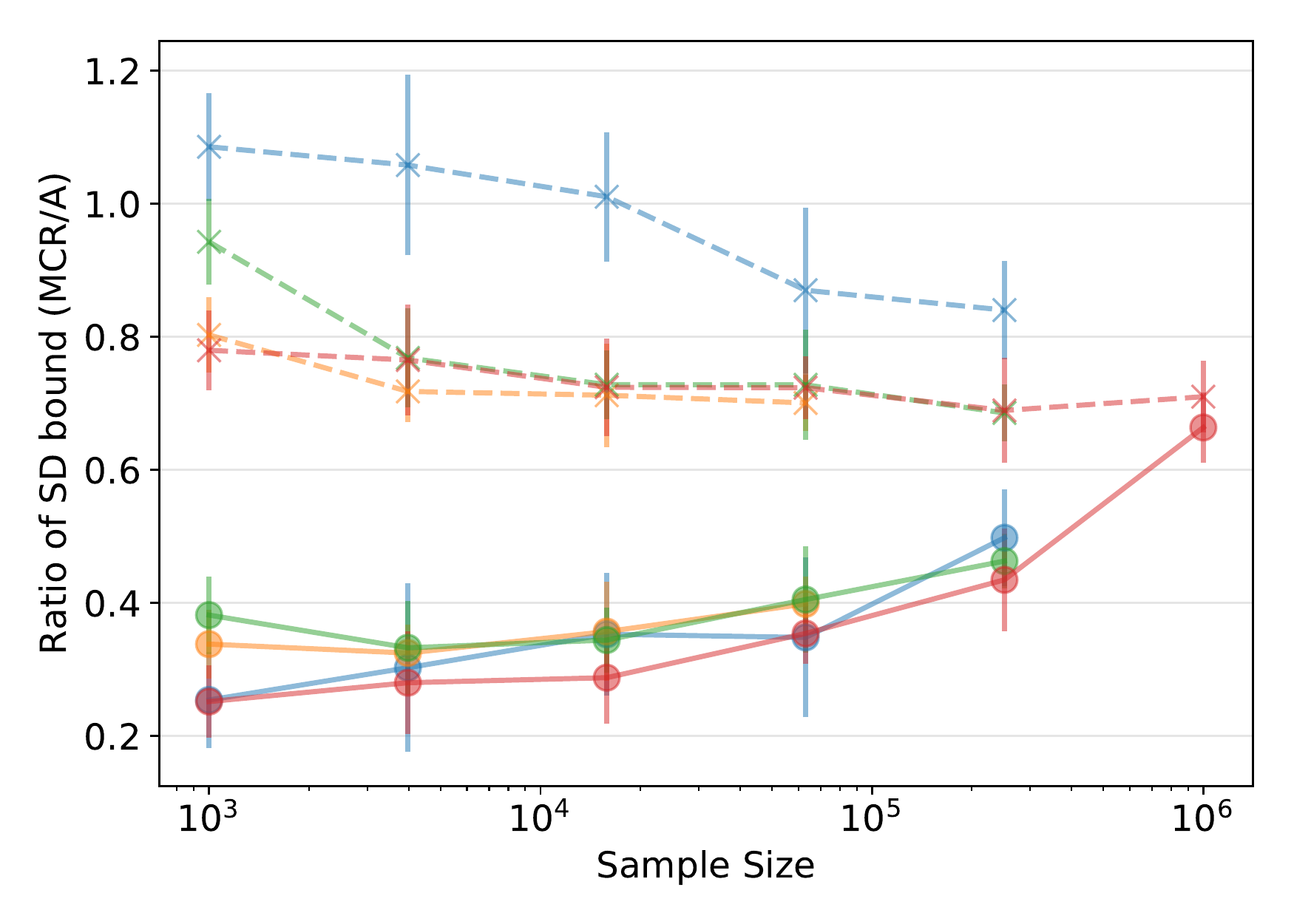}
  \caption{\iflongversion\else\vspace{-5pt}\fi$n=1$.}
\end{subfigure}%
\begin{subfigure}{.3\textwidth}
  \centering
  \includegraphics[width=\textwidth]{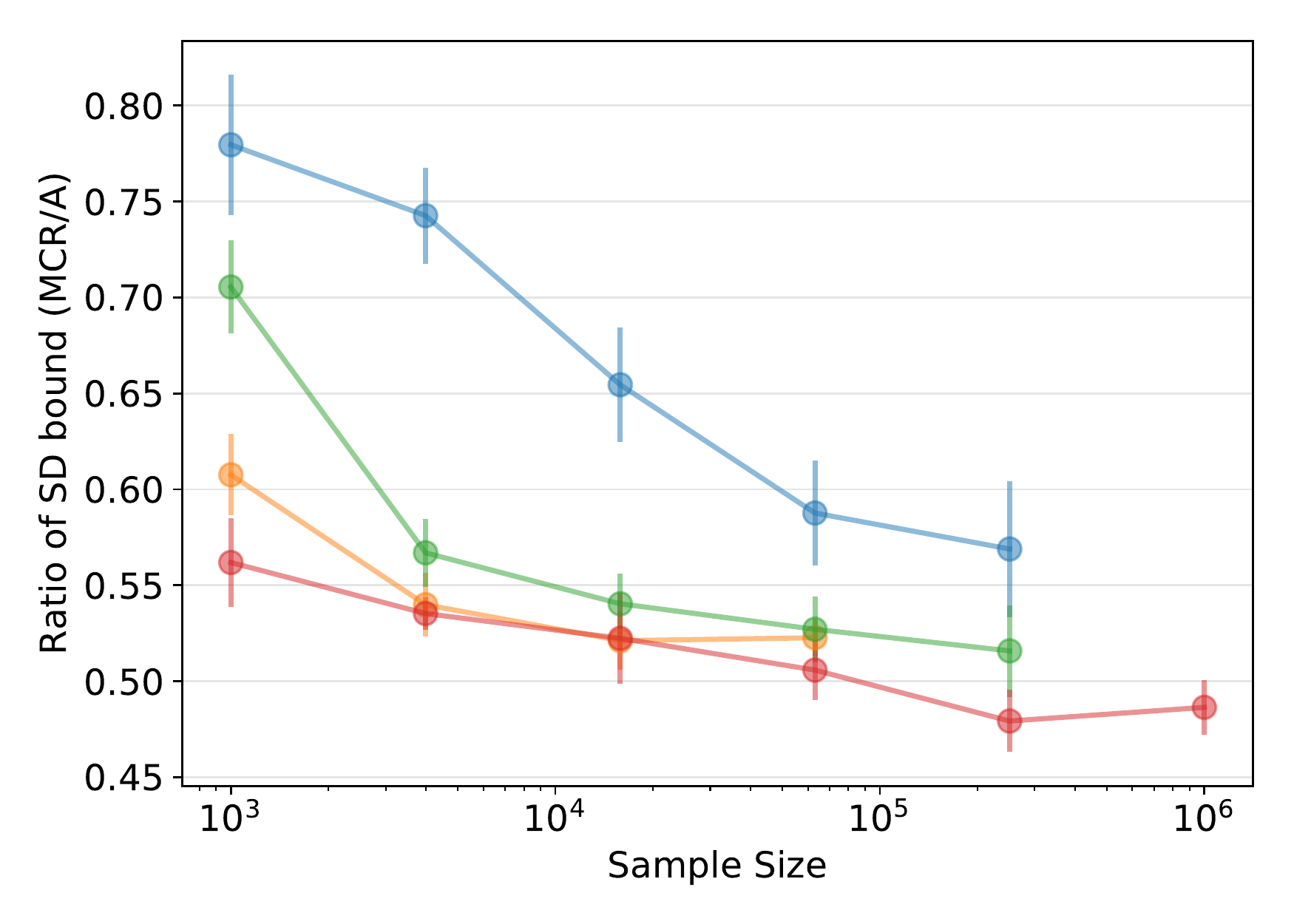}
  \caption{\iflongversion\else\vspace{-5pt}\fi$n=10$.}
\end{subfigure}
\begin{subfigure}{.3\textwidth}
  \centering
  \includegraphics[width=\textwidth]{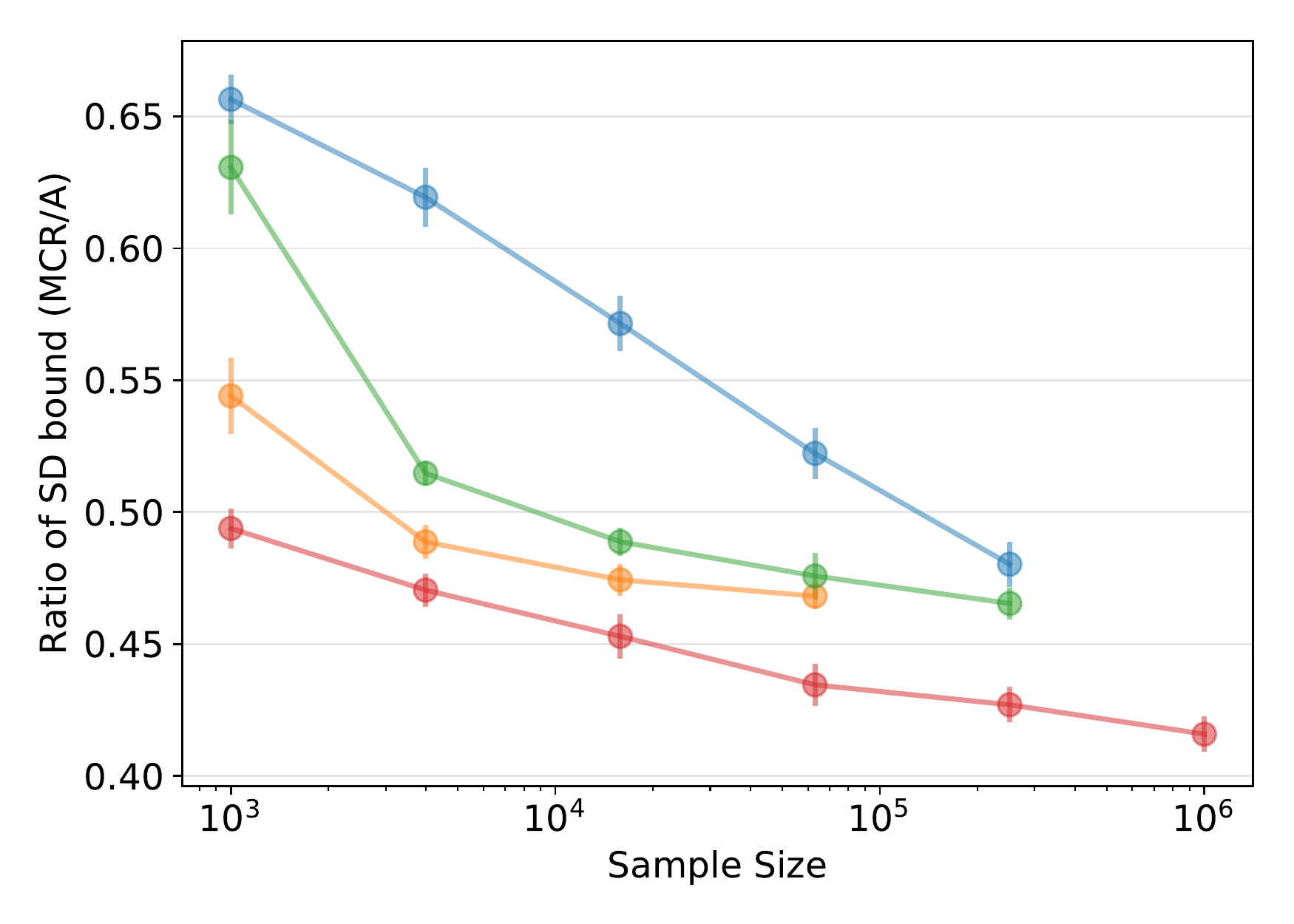}
  \caption{\iflongversion\else\vspace{-5pt}\fi$n=10^{2}$.}
\end{subfigure}
\iflongversion
\else
\vspace{-8pt}
\fi
\caption{Ratios of the SD Bound obtained by \algo\ (\mbox{$n\in \{ 1 ,
  10 , 10^{2} \}$}) and \amira\ for the entire $\F$, for $4$ of the datasets we
  analyzed. For $n=1$, dashed lines use the tail bound from
\cref{thm:supdevmcera} instead of the one
from~\cref{thm:supdev1mcera}.}\label{fig:epsilons}
\Description[Figure 1. The bound on the SD obtained by MCRapper is always smaller than the one computed by Amira.]{Figure 1: MCRapper computes an upper bound to the SD that is always smaller than Amira, for all sample sizes and for all values of n.}
\end{figure*}

\begin{figure*}[htb]
\iflongversion
\else
\vspace{-5pt}
\fi
\centering
\begin{subfigure}{.3\textwidth}
  \centering
  \includegraphics[width=\textwidth]{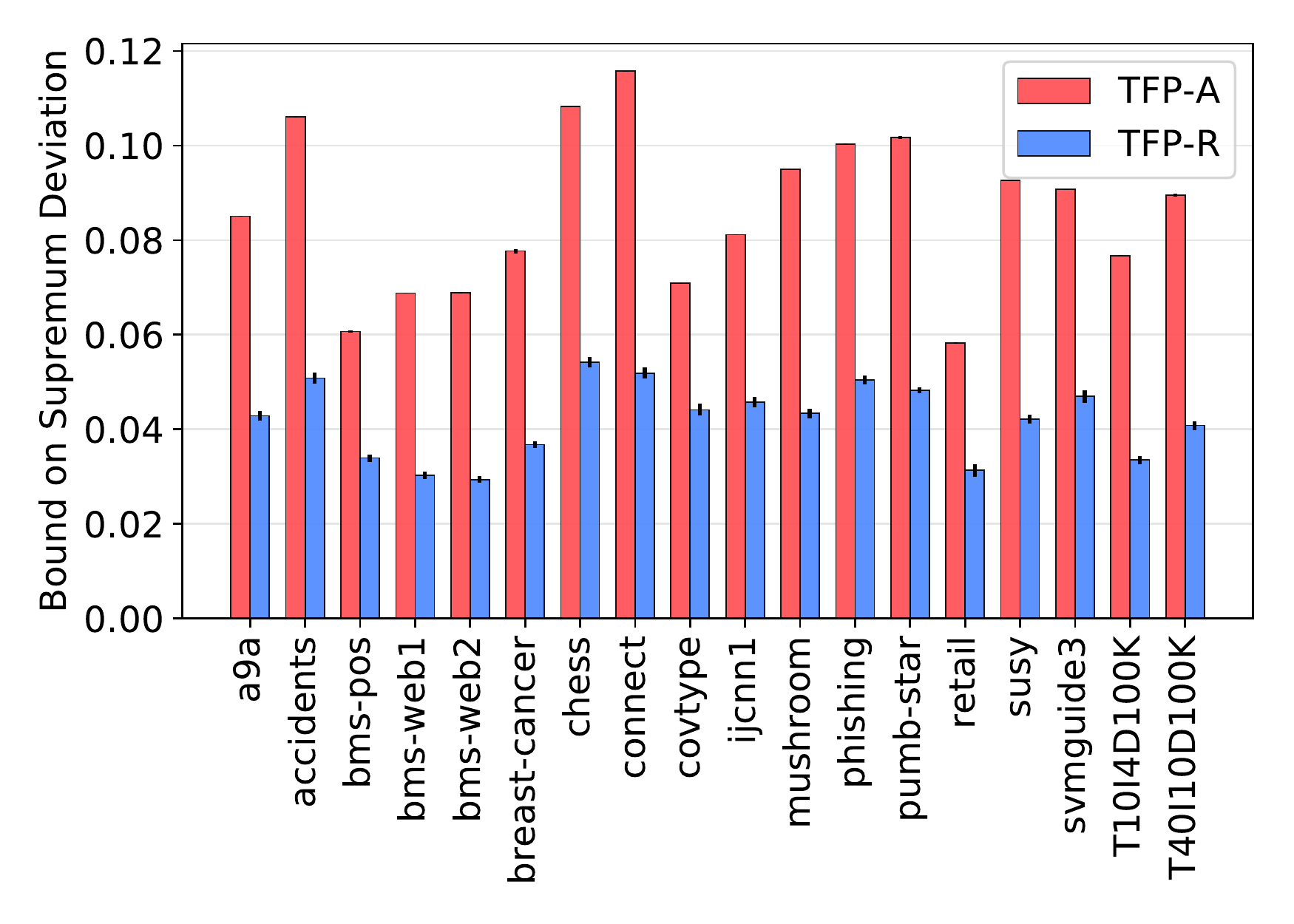}
  \caption{\iflongversion\else\vspace{-5pt}\fi}\label{fig:epsilons_tfp}
\end{subfigure}
\begin{subfigure}{.3\textwidth}
  \centering
  \includegraphics[width=\textwidth]{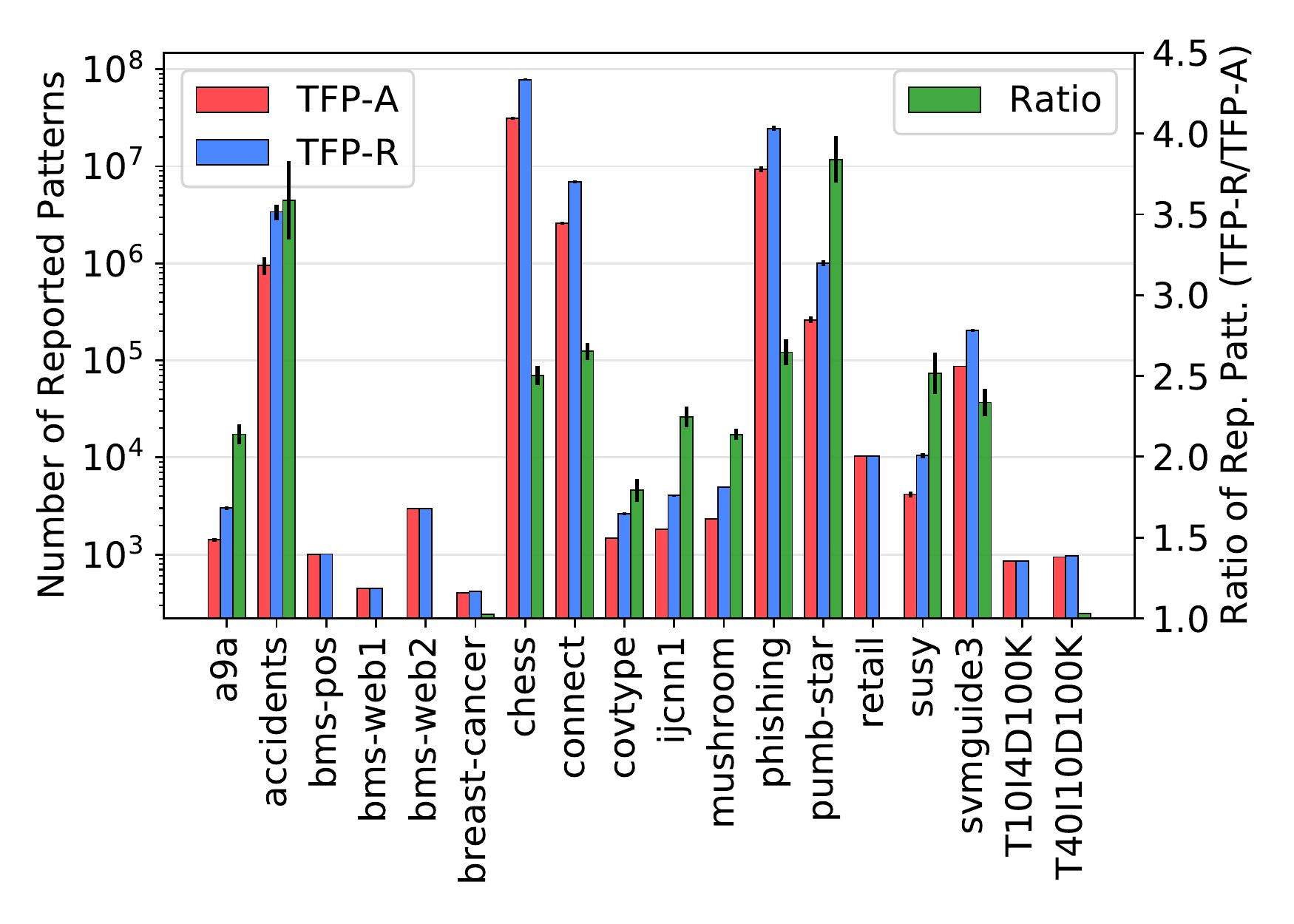}
\caption{\iflongversion\else\vspace{-5pt}\fi}\label{fig:num_tfp}\end{subfigure}
\begin{subfigure}{.3\textwidth}
  \centering
  \includegraphics[width=\textwidth]{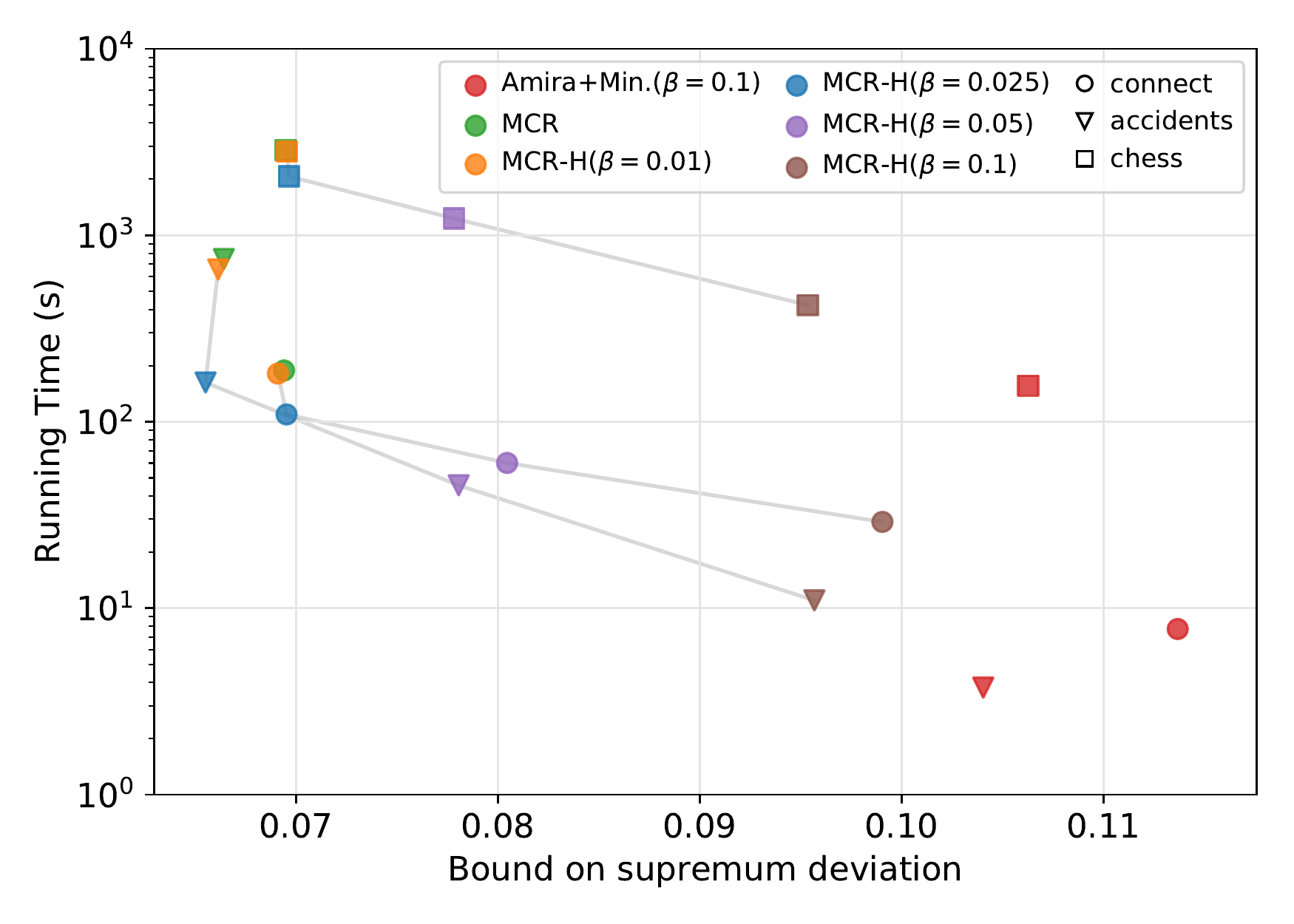}
  \caption{\iflongversion\else\vspace{-5pt}\fi}\label{fig:runningtimes_pareto1}
\end{subfigure}
\iflongversion
\else
\vspace{-8pt}
\fi
\caption{(a) Bound on the Supremum Deviation obtained by \algotfp\ and \algotfpamira.
  (b) Number of reported patterns (left $y$-axis) and ratios (right $y$-axis) by
\algotfp\ and \algotfpamira. (c) Running times of \algo, \algoh\ and \amira\ vs
corresponding upper bound on SD of the entire $\F$. For \algoh\ we use different
values of $\beta$. Each marker shape corresponds to one of the datasets we
considered (other $3$ shown in %
the Appendix%
). For \amira\ we also show the time for mining
the TFPs (\amira+Min.), with freq. $\geq \beta=0.1$, as needed after processing
the sample.}\label{fig:test}
\Description[Figure 2. TFP-R obtain a smaller upper bound on the SD w.r.t. TFP-A, and outputs an higher number of results. MCRapper-H allows to trade-off between the bound on the SD and running time.]{Figure 2: TFP-R always obtains a smaller upper bound on the SD w.r.t. TFP-A, resulting in a set of results that is always larger. The running time of MCRapper can be reduced using MCRapper-H, at the cost of computing a sligthly higher bound on the supremum deviation.}
\end{figure*}

\section{Applications}\label{sec:appl}
To showcase \algo's practical strengths, we now discuss applications to various
pattern mining tasks. The value $\varepsilon$ computed by \algo\ can be used,
for example, to compute, from a random sample $\sample$, a high-quality
approximation of the collection of frequent itemsets in a dataset \wrt\ a
frequency threshold $\theta \in (0,1)$, by mining the sample at frequency
$\theta - \varepsilon$~\citep{RiondatoU14}. Also, it can be used in the
algorithm by \citet{PellegrinaRV19a} to achieve statistical power in significant
pattern mining, or in the progressive algorithm by \citet{ServanSchreiberRZ18}
to enable even more accurate interactive data exploration. Essentially any of
the tasks we mentioned in \cref{sec:intro,sec:relwork} would benefit from the
improved bound to the SD computed by \algo. To support this claim, we now discuss
in depth one specific application.

\paragraph{Mining True Frequent Patterns}
We now show how to use \algo\ together with sharp variance-aware bounds to the
SD (\cref{thm:supdevvar}) for the specific application of identifying the
True Frequent Patterns (TFPs)~\citep{riondato2014finding}. The original work
considered the problem only for itemsets, but we solve the problem
for a general poset family of functions, thus for many other pattern classes,
such as sequences.

The task of TFP mining is, given a pattern language $\lang$ (i.e., a poset
family) and a threshold $\theta \in [0,1]$, to output the set
\[
  \tfp \left( \theta, \lang \right) = \left\lbrace f \in \lang :
    \E_{\probdist}[f] \geq \theta \right\rbrace \enspace.
\]
Computing $\tfp \left( \theta , \lang \right)$ \emph{exactly} requires to know
$\E_{\probdist}[f]$ for all $f$; since this is almost never the case (and in
such case the task is trivial), it is only possible to compute an
\emph{approximation} of $\tfp \left( \theta , \lang \right)$ using information
available from a random bag $\sample$ of $m$ \iid\ samples from $\probdist$. In
this work, mimicking the guarantees given in significant pattern
mining~\citep{HamalainenW18} and in multiple hypothesis testing settings, we are
interested in approximations that are a \emph{subset} of $\tfp(\theta, \lang)$,
i.e., we do not want false \emph{positives} in our approximation, but we accept
false \emph{negatives}. A variant that returns a superset of $\tfp(\theta,
\lang)$ is possible and only requires minimal modifications ot the algorithm.
Due to the randomness in the generation of $\sample$, no algorithm can guarantee
to be able to compute a (non-trivial) subset of $\tfp(\theta, \lang)$ from
\emph{every} possible $\sample$. Thus, one has to accept that there is a
probability over the choice of $\sample$ \emph{and other random choices made by
the algorithm} to obtain a set of patterns that is not a subset of $\tfp(\theta,
\lang)$. We now present an algorithm \algotfp\ with the following guarantee
(proof in %
\iflongversion%
\cref{sec:proofs}%
\else
the appendix of the extended online version at \refmissing%
\fi
).

\begin{theorem}\label{thm:correctness}
  Given $\lang$, $\sample$, $\theta \in [0,1]$, $\delta \in (0, 1)$, and a
  number $n\ge 1$ of Monte-Carlo trials, \algotfp\ returns a set
  $Y$ such that
  \[
    \Pr_{\sample,\vsigma}( Y \subseteq \tfp(\theta, \lang) ) \ge 1 - \delta,
  \]
  where the probability is over the choice of both $\sample$ and the randomness
  in \algotfp, i.e., an $n\times m$ matrix of \iid\ Rademacher variables
  $\vsigma$.
\end{theorem}

The intuition for \algotfp\ is the following. Let
$\negborder(\tfp(\theta, \lang))$ be the \emph{negative border} of $\tfp(\theta,
\lang)$, that is, the set of functions in $\lang \setminus \tfp(\theta, \lang)$
such that every parent \wrt\ $\preceq$ of $f$ is in $\tfp(\theta, F)$. If we can
compute an $\hat{\varepsilon} \in (0,1)$ such that, for every $f \in
\negborder(\tfp(\theta, \lang))$, it holds $\EE_\sample[f] \le \theta +
\hat{\varepsilon}$, then we can be sure that any $g \in \lang$ such that
$\EE_\sample[g] > \theta + \hat{\varepsilon}$ belongs to $\tfp(\theta, \lang)$.
This guarantee will naturally be probabilistic, for the reasons we already
discussed.  Since $\negborder(\tfp(\theta, \lang))$ is unknown, \algotfp\
approximates it by \emph{progressively} refining a \emph{superset} $\tempset$ of
it, starting from $\lang$. The correctness of \algotfp\ is based on the fact
that at every point in the execution, it holds $\negborder(\tfp(\theta, \lang))
\subseteq \tempset$, as we show in the proof of \cref{thm:correctness}.

\begin{algorithm}[htb]
  \SetNoFillComment%
  \DontPrintSemicolon
  \KwIn{Poset family $\lang$, sample $\sample$ of size $m$, $\theta \in [0,1]$,
  $\delta \in (0,1)$, $n \geq 1$.}
  \KwOut{A set $Y$ of patterns}
  \SetKwFunction{GetSupDevBoundVar}{getSupDevBoundVar}
  $Y \gets \emptyset$\;
  $\vsigma \gets $ \Draw{$m$, $n$}\label{algline:matrix}\;
  \lIf*{$\theta \ge \tfrac{1}{2}$}{$v \gets \tfrac{1}{4}$} \lElse{$v \gets \theta
  ( 1 - \theta)$\label{algline:variance}}
  $\tempset \gets \lang$\label{algline:tempsetinit}\;
  \Repeat{$\tempset = \tempset'$}{\label{algline:startloop}
    $\hat{\varepsilon} \gets $ \GetSupDevBoundVar{$\tempset$, $\sample$,
    $\delta$, $\vsigma$, $v$}\label{algline:eps}\;
    $\tempset' \gets \tempset$\;
    $\tempset \gets \{ f \in \tempset'\ :\ \EE_{\sample}[f] < \theta +
    \hat{\varepsilon} \}$\label{algline:filtering}\;
    $Y \gets Y \cup (\tempset' \setminus \tempset)$\label{algline:outadd}\;
  }\label{algline:endloop}
  \Return{$Y$}
  \caption{\algotfp}\label{algo:tfp}
\end{algorithm}

The pseudocode of \algotfp\ is presented in \cref{algo:tfp}. The algorithm
first draws the matrix $\vsigma$ (line~\ref{algline:matrix}), and then computes
an upper bound $v$ to the variances of the the frequencies in
$\negborder(\tfp(\theta, \lang))$ (line~\ref{algline:variance}). It then
initializes, as discussed above, the set $\tempset$ to $\lang$
(line~\ref{algline:tempsetinit}) and enters a loop. At each iteration of the
loop, \algotfp\ calls the function \GetSupDevBoundVar\ which returns a value
$\hat{\varepsilon}$ computed as in~\eqref{eq:supdevvar-eps} using $\F =
\tempset$, and $\eta = \delta$. The function \GetNMCERA\ from
\cref{algo:main} is used inside of \GetSupDevBoundVar\ (with parameters
$\tempset$, $\sample$, and $\vsigma$) to compute the $n$-MCERA in the value
$\rho$ from~\eqref{eq:supdevvar-rho}. The properties of $\hat{\varepsilon}$ are
discussed in the proof for \cref{thm:correctness}.


\algotfp\ uses $\hat{\varepsilon}$ to refine the set $\tempset$ with the goal of
obtaining a better approximation of $\negborder(\tfp(\theta, \lang))$. The set
$\tempset'$ stores the current value of $\tempset$, and the new value of
$\tempset$ is obtained by keeping all and only the patterns $f \in \tempset'$
such that $\EE_\sample[f] < \theta + \hat{\varepsilon}$
(line~\ref{algline:filtering}). All the patterns that have been filtered out,
i.e., the patterns in $\tempset' \setminus \tempset$, or in other words, all the
patterns $f \in \tempset'$ such that $\EE_\sample[f] \ge \theta +
\hat{\varepsilon}$, are added to the output set $Y$ (line~\ref{algline:outadd}).
\algotfp\
keeps iterating until the value of $\tempset$ does not change from the previous
iteration (condition on line~\ref{algline:endloop}), and finally the set $Y$ is
returned in output. While we focused on the a conceptually high-level description of \algotfp, we note that an efficient implementation only requires \emph{one} exploration of $\F$, such that $Y$ can be provided in output \emph{as $\F$ is explored}, therefore without executing either multiple instances of \algo\ or, at the end of \algotfp, a frequent pattern mining algorithm to compute $Y$.


\section{Experiments}\label{sec:experiments}
In this section we present the results of our experimental evaluation for \algo.
We compare \algo\ to \amira~\citep{RiondatoU15}, an algorithm that bounds the
Supremum Deviation by computing a deterministic upper bound to the ERA with one
pass on the random sample. The goal of our experimental evaluation is to compare
\algo\ to \amira\ in terms of the upper bound to the SD they compute. We also
assess the impact of the difference in the SD bound provided by \algo\ and
\amira\ for the application of mining true frequent patterns, by comparing our
algorithm \algotfp\ with \algotfpamira, a simplified variant of \algotfp\ that
uses \amira\ to compute a bound $\varepsilon$ on the SD for all functions in
$\lang$, and returns as candidate true frequent patterns the set
$\freqset(\theta + \varepsilon , \sample)$. It is easy to prove that the output
of \algotfpamira\ is a subset of true frequent patterns with probability $\ge 1
- \delta$.  We also evaluate the running time of \algo\ and of its variant
\algoh.

\paragraph{Datasets and implementation}
We implemented \algo\ and \algoh\ in \texttt{C}, by modifying
\topkwy~\cite{pellegrina2018efficient}. Our implementations are available at
\repolink.  The implementation of \amira~\cite{RiondatoU15} has been provided
by the authors. We test both methods on 18 datasets (see \cref{tab:data} in the
Appendix for their statistics), widely used for the benchmark of
frequent itemset mining algorithms. To compare \algo\ to \amira\ in terms of the
upper bound to the SD, we draw, from every dataset, random samples of increasing
size $m$; we considered $6$ values equally spaced in the logarithmic space in
the interval $[10^{3} , 10^{6}]$. We only consider values of $m$ smaller than
the dataset size $|\D|$. For both algorithms we fix $\delta = 0.1$. For \algo\,
we use $n \in \{1 , 10 , 100 \}$.

To compare \algotfp\ to \algotfpamira, we analyze synthetic datasets of size
$m=10^{4}$ obtained by random sampling transactions from each dataset: the true
frequency of a pattern corresponds to its frequency in the original dataset,
which we use as the ground truth. We use $n=10$ for \algotfp, and $\delta=0.1$.
We report the results for $\theta = 0.05$ (other values of $\theta$ and $n$
produced similar results).

For all experiments and parameters combinations we perform $10$ runs (i.e., we
create $10$ random samples of the same size from the same dataset). In all the
figures we report the averages and avg $\pm$ standard deviations of these runs.

\subsection{Bounds on the SD}\label{sec:results1}
\Cref{fig:epsilons} shows the ratio between the upper bound on the SD obtained
by \algo\ and the one obtained by \amira\ for different values of $n$. The bound
provided by \algo\ is always better (i.e., lower) than the bound provided
by \amira\ (e.g., for $n=100$ the bound from \algo\ is always at least $34\%$
smaller than the bound from \amira). For $n=1$ one can see that the
\emph{novel} improved bound from \cref{thm:supdev1mcera} should really be
preferred over the ``standard'' one (dashed lines). Similar results hold for all
other datasets. These results highlight the effectiveness of \algo\ in providing
a much tighter bound to the SD than currently available approaches.

\subsection{Mining True Frequent Patterns}\label{sec:results2}
We compare the \emph{final} SD computed by \algo\ with the one computed by
\algotfpamira. The results are shown in \cref{fig:epsilons_tfp}. Similarly to what we
observed in \cref{sec:results1}, \algo\ provides much tighter bounds being, in
most cases, less than $50\%$ of the bound reported by \amira. We then assessed
the impact of such difference in the mining of TFP, by
comparing the number of patterns reported by \algotfp\ and by \algotfpamira.
Since for both algorithms the output is a subset of the true frequent patterns
with probability $\ge 1-\delta$, reporting a higher number of patterns
corresponds to identifying more true frequent patterns, i.e., to higher power.
\Cref{fig:num_tfp} shows the number of patterns reported by \algotfp\ and by
\algotfpamira\ (left $y$-axis) and the ratio between such quantities (right
$y$-axis). The SD bound from \algo\ is always lower than the SD bound from
\amira, so \algotfp\ always reports at least as many patterns as \algotfpamira,
and for 10 out of 18 datasets, it reports at least \emph{twice} as many
patterns as \algotfpamira. These results show that the SD bound computed by
\algotfp\ provides a great improvement in terms of power for mining TFPs \wrt\
current state-of-the-art methods for SD bound computation.

%
%

%
%

\iflongversion
\else
\vspace{-5pt}
\fi
\subsection{Running time}\label{sec:hybridresults}
For these experiments we take $10$ random samples of size $10^{4}$ of the $6$
most demanding datasets (\texttt{accidents}, \texttt{chess}, \texttt{connect},
\texttt{phishing}, \texttt{pumb-star}, \texttt{susy}; for the other datasets
\algo\ takes much less time than the ones shown) and use the hybrid
approach \algoh\ (\cref{sec:hybridalgorithm}) with different
values of $\beta$ (and $n=1$, which gives a good trade-off between the bounds
and the running time, $\gamma = 0.01$, $\delta = 0.1$). %
We
na\"{\i}vely upper bound  $|\infreqset(\sample, \beta)|$ with
$\sum_{s_{i \in \sample}}2^{|s_{i}|}$, where $|s_{i}|$ is the length of the
transaction $s_{i}$, a \emph{very loose} bound that could be improved using more
information from $\sample$. %
\Cref{fig:runningtimes_pareto1,fig:runningtimes_pareto} (in the Appendix) show
the running time of \algo\ and \amira\ vs.\ the obtained upper bound on the SD
(different colors correspond to different values of $\beta$). With \amira\ one
can quickly obtain a fairly loose bound on the SD, by using \algo\ and \algoh\
one can trade-off the running time for smaller bounds on the SD.%

\iflongversion
\else
\vspace{-4pt}
\fi
\section{Conclusion}\label{sec:concl}
We present \algo, an algorithm for computing a bound to the supremum deviation
of the sample means from their expectations for families of functions with
poset structure, such as those that arise in pattern mining tasks. At the
core of \algo\ there is a novel efficient approach to compute the $n$-sample
Monte-Carlo Empirical Rademacher Average based on fast search space exploration
and pruning techniques. %
\algo\ returns a much better (i.e., smaller) bound to the supremum deviation than
existing techniques. We use \algo\ to extract true frequent patterns and show
that it finds many more patterns than the state of the art.

\iflongversion
\else
\vspace{-6pt}
\fi
\begin{acks}
\iflongversion
\else
\vspace{-3pt}
\fi
Part of this work was conducted while L.P.~was visiting the Department of
Computer Science of Brown University, supported by a
``\grantsponsor{gini}{Fondazione Ing.\ Aldo
Gini}{https://www.unipd.it/fondazionegini}'' fellowship. Part of this work is
supported by the \grantsponsor{NSF}{National Science
Foundation}{https://www.nsf.gov} grant
\grantnum{NSF}{RI-1813444}, by the \grantsponsor{miur}{MIUR of
Italy}{http://www.miur.it} under \grantnum{miur}{PRIN Project n. 20174LF3T8}
AHeAD (Efficient Algorithms for HArnessing Networked Data), and by the
\grantsponsor{unipd}{University of
Padova}{http://www.unipd.it} grant \grantnum{unipd}{STARS 2018}.
\end{acks}

\iflongversion%
\bibliographystyle{plainnat}
\bibliography{bibliography}
\else%
\vspace{-6pt}

\fi%

\clearpage
\appendix
\section{Appendix}\label{sec:appendix}

\subsection{Missing Proofs}\label{sec:proofs}

\iflongversion%
\begin{theorem}[Symmetrization inequality
  \lbrack\citealp{KoltchinskiiP00a}\rbrack]\label{thm:symmetrization}
  For any family $\F$ it holds $\E_\sample \left[ \sup_{f \in \F} \left( \EE_\sample[f] -
    \E_\probdist[f] \right) - 2 \erade(\F, \sample) \right] \le 0$
\end{theorem}
\begin{theorem}[\lbrack\citealp{Abousquet2002bennett},
  Thm.~2.2\rbrack]\label{thm:sdvarbound}
  Let $Z = \sup_{f \in \F} \left( \EE_\sample[f] -
  \E_\probdist[f] \right)$. Let $\eta \in (0, 1)$. Then, with probability at
  least $1 - \eta$ over the choice of $\sample$, it holds
  \begin{equation}\label{eq:sdvarbound}
    Z \le \E_{\probdist}\left[ Z \right] + \sqrt{\frac{2 \ln \frac{1}{\eta}
    \left( v + 4\frange\E_{\probdist}[ Z ] \right)}{m}}
        + \frac{2 \frange \ln \frac{1}{\eta}}{3m} \enspace.
  \end{equation}
\end{theorem}

\begin{proof}[Proof of \cref{thm:supdevvar}]
  Consider the following events
  \begin{align*}
    \mathsf{E}_1 \doteq \rho \ge \era,&\\
    \mathsf{E}_2 \doteq E_\probdist[\era] &\le \era\\
                 & + \frac{1}{2m} \left (\sqrt{\frange \left( 4m
            \rho + \frange \ln \frac{4}{\delta} \right) \ln
      \frac{4}{\delta}} + \frange \ln \frac{4}{\delta} \right) \enspace.
  \end{align*}

  From \cref{lem:mcera}, we know that $\mathsf{E}_1$ holds with
  probability at least $1 - \tfrac{\delta}{4}$ over the choice of $\sample$ and
  $\vsigma$.  
  $\mathsf{E}_2$ is guaranteed to with probability at least $1 -
  \tfrac{\delta}{4}$ over the choice of $\sample$~\citep[(generalization of)
  Thm.~3.11]{AOnetoGAR13}. 
  Define the event $\mathsf{E}_3$ as the event
  in~\eqref{eq:sdvarbound} for $\eta =
  \tfrac{\delta}{4}$ and the event $\mathsf{E}_4$ as the event
  in~\eqref{eq:sdvarbound} for $\eta = \tfrac{\delta}{4}$ and for $\F = -
  \mathcal{F}$. \citep[Thm.~2.2]{Abousquet2002bennett} 
  tells us that events $\mathsf{E}_3$ and
  $\mathsf{E}_4$ hold each with probability at least $1-\tfrac{d}{4}$ over the
  choice of $\sample$. Thus from the union bound we
  have that the event $\mathsf{E} = \mathsf{E}_1 \cap \mathsf{E}_2 \cap
  \mathsf{E}_3$ holds with probability at least $1 - \delta$ over the choice of
  $\sample$ and $\vsigma$. Assume from now on that the event $\mathsf{E}$ holds.

  Because $\mathsf{E}$ holds, it must be $r \ge \E_\probdist[\erade(\F,
  \sample)]$. From this result and \cref{thm:symmetrization} we have that
  \[
    \E_\probdist[\sup_{f \in \F} \left( \EE_\sample[f] - \E_\probdist[f]
    \right)] \le 2 \E_\probdist[\erade(\F, \sample)] \le 2 r \enspace.
  \]
  From here, and again because $\mathsf{E}$, by plugging $2 r$ in place of
  $E[Z]$ into~\eqref{eq:sdvarbound} (for $\eta=\tfrac{\delta}{4}$), we obtain
  that $\sup_{f \in \F} \left( \EE_\sample[f] - \E_\probdist[f] \right) \le
  \varepsilon$. To show that it also holds
  \[
    \sup_{f \in \F} \left( \EE_\sample[f] - \E_\probdist[f] \right) \le
    \varepsilon
  \]
  (which allows us to conclude that $\supdev \le \varepsilon$), we repeat the
  reasoning above for $-\F$ and use the fact that $\erade(\F, \sample) =
  \erade(-\F, \sample)$, a known property of the ERA, thus
  \begin{align*}
    \rho &\ge \erade(-\F, \sample)\ \text{and}\ r \ge
    E_\probdist[\erade(-\F, \sample)]\ \text{and}\\
    \varepsilon &\ge \sd(-\F, \sample) = \sup_{f
  \in \F} \left( \EE_\sample[f] - \E_\probdist[f] \right) \enspace. \qedhere
    \end{align*}
\end{proof}

\begin{theorem}[McDiarmid's
  inequality~\lbrack\citealp{Amcdiarmid1989method}\rbrack]\label{thm:mcdiarmid}
  Let $\mathcal{Y} \subseteq \R^\ell$, and let $g : \mathcal{Y} \rightarrow \R$
  be a function such that, for each $i$, $1\le i\le \ell$, there is a
  nonnegative constant $c_i$ such that:
  \begin{equation}\label{eq:bounded}
    \sup_{\substack{x_1,\dotsc,x_\ell\\x_i'\in\mathcal{X}}}|g(x_1,\dotsc,x_\ell)-g(x_1,\dotsc,x_{i-1},x'_i,x_{i+1},\dotsc,x_\ell)|\le
    c_i\enspace.
  \end{equation}
  Let $x_1,\dotsc,x_\ell$ be $\ell$ \emph{independent} random variables taking
  value in $\R^\ell$ such that $\langle x_1,\dotsc,x_\ell\rangle \in
  \mathcal{Y}$. Then it holds
  \[
    \Pr\left(g(x_1,\dotsc,x_\ell) -\E_\probdist[g]> t\right)\le e^{-2t^2/C},
  \]
  where $C=\sum_{i=1}^\ell c_i^2$.
\end{theorem}
The following result is an application of McDiarmid's inequality to the
$n$-MCERA, with constants $c_i = \tfrac{2\rmax}{nm}$.
\begin{lemma}\label{lem:mcera}
  Let $\eta \in (0,1)$. Then, with probability at least $1 - \eta$ over the
  choice of $\vsigma$, it holds
  \[
    \era = \E_{\vsigma} \left[ \erade^{n}_{m}(\F, \sample, \vsigma) \right] \leq
    \erade^{n}_{m} (\F, \sample, \vsigma) + 2\rmax \sqrt{ \frac{\ln
    \frac{1}{\eta}}{2nm}} \enspace.
  \]
\end{lemma}

The following result gives a probabilistic upper bound to the supremum
deviation using the RA and the ERA~\citep[Thm.~3.11]{AOnetoGAR13}.

\begin{theorem}\label{thm:supdev}
  Let $\eta \in (0, 1)$. Then, with probability at least $1 - \eta$ over the
  choice of $\sample$, it holds
  \begin{align}
    &\supdev \leq 2 \era \nonumber\\
    & + \frac{\sqrt{\frange \left( 4m \era + \frange \ln \frac{3}{\eta}
    \right) \ln \frac{3}{\eta}}}{m} + \frac{\frange\ln \frac{3}{\eta}}{m} +
    \frange\sqrt{\frac{\ln \frac{3}{\eta}}{2m}}
    \enspace.\footnotemark\label{eq:supdevera}
  \end{align}
  \footnotetext{Slightly sharper bounds are possible at the expense of an
  increased complexity of the terms.}
\end{theorem}

\begin{proof}[Proof of \cref{thm:supdevmcera}]
  Through \cref{lem:mcera} (using $\eta$ there equal to $\tfrac{\eta}{4}$),
  \cref{thm:supdev} (using $\eta$ there equal to $\tfrac{3\eta}{4}$), and an
  application of the union bound.
\end{proof}

\begin{proof}[Proof of \cref{thm:discrbounds}]
  It is immediate from the definitions of $\ndbound{f}{j}$ and $\dbound{f}{j}$
  in~\eqref{eq:bounds} that $\dbound{f}{j} \leq \ndbound{f}{j}$, so we can focus
  on $\dbound{f}{j}$. We start by showing that $\discr{f}{j} \leq
  \dbound{f}{j}$. It holds
  \begin{align*}
    &\discr{f}{j} =\\
    &\sum_{s_{i} \in \sample} \indp f^+(s_i) - \sum_{s_{i}
    \in \sample} \indm f^-(s_i) - \sum_{s_{i} \in \sample} \indm f^+(s_i)
    + \sum_{s_i \in \sample} \indp f^-(s_i) \\
    & \leq \sum_{s_{i} \in \sample} \indp f^+(s_i) - \sum_{s_{i}
    \in \sample} \indm f^-(s_i) = \dbound{f}{j}
  \end{align*}
  where the inequality comes from the fact that $\sum_{s_i \in \sample} \indm
    f^+(s_i) \ge 0$, and $\sum_{s_i \in \sample} \indp f^-(s_i) \le 0$.

  To prove that $\discr{g}{j} \le \dbound{f}{j}$ for every $g \in
  \desc(f)$ it is sufficient to show that $\dbound{g}{j} \leq \dbound{f}{j}$
  hold for every such $g$, since we just showed that $\discr{g}{j} \leq
  \dbound{g}{j}$ is true for any $f \in \F$. It holds $f \preceq g$, so from the
  definition of the relation $\preceq$ in~\eqref{eq:preceq}, we get
  \begin{align*}
    \dbound{g}{j} &= \sum_{s_{i} \in \sample} \indp g^+(s_i) - \sum_{s_{i} \in
    \sample} \indm g^-(s_i) \\
    &\le \sum_{s_{i} \in \sample} \indp f^+(s_i) - \sum_{s_{i} \in
    \sample} \indm f^-(s_i) = \dbound{f}{j}
    \end{align*}
  which completes our proof.
\end{proof}

\begin{proof}[Proof of \cref{lem:getNMCERA}]
  For $j \in \{1, \dotsc, n\}$, let $h_j$ be any of the functions attaining the
  supremum in $\sup_{f \in f} \discr{f}{j}$. 
  We need to show that the algorithm updates $\nu_j$ on
  line~\ref{algline:discr} of \cref{algo:main} using $\discr{h_j}{j}$ at some
  point during its execution. We focus on a single $j$, as the proof is the same
  for any value of $j$.

  It is evident from the description of the algorithm that $\nu_j$ is always
  only set to values of $\discr{g}{j}$, and since $h_j$ has the maximum of these
  values, $\nu_j$ will be, at any point in the execution of the algorithm less
  than or equal to $\discr{h_j}{j}$. Let's call this fact $\mathsf{F}_1$. Thus,
  if the algorithm ever hits line~\ref{algline:discr} with $f = h_j$, then we
  can be sure that the value stored in $\nu_j$ will be $\discr{h_j}{j}$, and
  this variable will never take an higher value. From fact $\mathsf{F}_1$ and
  \cref{thm:discrbounds} we also have that at any point in time it must be
  $\nu_j \le \dbound{h_j}{j} \le \ndbound{h_j}{}$, so the conditions on
  lines~\ref{algline:dbound} and~\ref{algline:candloop} are definitively
  satisfied, so the question is now whether $j \in \mathcal{J}[h_j]$ and whether
  there is an iteration of the \textbf{while} loop of
  line~\ref{algline:queueloop} for which $f = h_j$.

  It holds from \cref{thm:discrbounds} that it must be $\discr{h_j}{j} \le
  \dbound{g}{j} \le \ndbound{g}{}$ for every ancestor $g$ of $h_j$. From this
  fact and from fact $\mathsf{A}$ then it holds that at any point in time it
  must hold $\nu_j \dbound{g}{j} \le \ndbound{g}{}$ for every such ancestor $g$
  of $h_j$. Thus, the value $j$ is always added to the set $Y$ at every
  iteration of the \textbf{while} loop for which $f$ is an ancestor of $h_j$.
  Let's call this fact $\mathsf{F}_2$. Thus, as long as no ancestor of $h_j$ is
  pruned or $h_j$ itself is pruned, $j$ is guaranteed to be in
  $\mathcal{J}[h_j]$. But from fact $\mathsf{F}_2$ and from the fact that $j$
  belongs to $\mathcal{J}[f]$ for all the ancestors of $h_j$ that are in
  \Minimals{$f$} (line~\ref{algline:candsminimals}), then $j$ must belong to the
    set $N$ computed on line~\ref{algline:childCands} for all ancestors of
    $h_j$, thus $N$ is never empty and therefore no ancestor of $h_j$ is ever
    pruned and neither is $f$ and we are guaranteed that $h_j$ is added to $Q$
    on line~\ref{algline:push} when the first of its parents is visited. Thus,
    there is an iteration of the \textbf{while} loop that has $f=h_j$, and
    because of what we discussed above, then it will be the case that $\nu_j =
    \discr{h_j}{j}$ and our proof is complete.
\end{proof}

\begin{proof}[Proof of \cref{thm:supdev1mcera}]
  For ease of notation, let $\mathcal{G} = \F - \tfrac{\frange}{2}$. Consider
  the event
  \begin{equation}\label{eq:supdev1mcera-tech}
    \mathsf{E}_1 \doteq \sup_{g \in \mathcal{G}} \left( \EE_\sample[g] -
    \E_\probdist[g] \right) \le 2 \erade_{m}^1 (\mathcal{G}, \sample, \vsigma) +
    3 \frange \sqrt{\frac{\ln \frac{2}{\eta}}{2m}} \enspace.
  \end{equation}
  We now show that this event holds with probability at least $1-
  \tfrac{\eta}{2}$ over the choices of $\sample$ and $\vsigma$, and then we use
  this fact to obtain the thesis with some additional steps.

  Using linearity of expectation and the fact that the $n$-MCERA is an unbiased
  estimator for the ERA (i.e., its expectation is the ERA), we can rewrite the
  symmetrization inequality (\cref{thm:symmetrization}) as
  \[
    \E_{\sample, \vsigma}\left[ \sup_{g \in \mathcal{G}} \left( \EE_\sample[g] -
      \E_\probdist[g] \right) - 2\erade^{1}_{m}(\mathcal{G}, \sample, \vsigma)
    \right] \leq 0 \enspace.
  \]
  The argument of the (outmost) expectation on the l.h.s.~can be seen as a
  function $h$ of the $m$ pairs of r.v.'s $(\vsigma_{1,1}, s_1), \dotsc,
  (\vsigma_{1,m}, s_m)$. Fix any possible assignment $v'$ of values to these
  pairs. Consider now a second assignment $v''$ obtained from $v'$  by changing
  the value of \emph{any of the pairs} with any other value in $\{-1, 1\} \times
  \X$. We want to show that it holds $|h(v') - h(v'') | \le
  3\tfrac{\frange}{m}$.

  We separately handle the SD and the $1$-MCERA, as both depend on the values of
  the assignment of values to the pairs. The SD does not depend on
  $\vsigma_{1,\cdot}$, and in the argument of the supremum, changing any
  $s_{j}$ changes a single summand of the empirical mean $\EE_\sample[f]$, with
  maximal change when $f(s_{j})$ changes from $a$ to $b$ (or viceversa), thus
  the SD itself changes by no more than $\tfrac{c}{m}$.

  We now consider the $1$-MCERA, and assume that the pair changing value is
  $(\vsigma_{1,j}, s_j)$. Then the only term of the 1-MCERA sum that changes is
  the $j$-th term. If only the first component of the pair changes value (i.e.,
  $\vsigma_{1,j}$ changes from $1$ to $-1$ or viceversa, i.e., from
  $\vsigma_{1,j}$ to $-\vsigma_{1,j}$), then the $j$-th term in the 1-MCERA sum
  cannot change by more than $\frange$, because it holds $\vsigma_{1,j} g(s_j)
  \in [-\tfrac{\frange}{2}, \tfrac{\frange}{2}]$, thus $-\vsigma_{1,j} g(s_j)$
  also belongs to this interval, and it must be $|\vsigma_{1,j} g(s_j) -
  (-\vsigma_{1,j} g(s_j))| \le \frange$.  If only the second component of the
  pair changes value (i.e., $s_j$ changes value to $\bar{s}_j$), then the $j$-th
  term in the 1-MCERA sum cannot change by more than $\frange$, because each
  function $g \in \mathcal{G}$ goes from $\X$ to $[-\tfrac{\frange}{2},
  \tfrac{\frange}{2}]$, and it must be $|\vsigma_{1,j} g(s_j) - \vsigma_{i,j}
  g(\bar{s}_j)| \le \frange$. Consider now the final case where both
  $\vsigma_{1,j}$ and $s_j$ change value.  We have once again $|\vsigma_{1,j}
  g(s_j) - (-\vsigma_{1,j} g(\bar{s}_j))| \le \frange$.

  By the adding the maximum change in the SD and the maximum change in the
  1-MCERA we can conclude that function $h$ satisfies the requirements of
  McDiarmid's inequality (\cref{thm:mcdiarmid}) with constants
  $3\tfrac{\frange}{m}$, and obtain that event $\mathsf{E}_1$
  from~\eqref{eq:supdev1mcera-tech} holds with probability at least $1 -
  \tfrac{\eta}{2}$.

  Let now $-\mathcal{G}$ represent the family of functions containing $-g$ for
  each $g \in \mathcal{G}$. Consider the event
  \[
    \mathsf{E}_2 \doteq \sup_{g \in -\mathcal{G}} \left( \EE_\sample[g] -
    \E_\probdist[g] \right) \le 2 \erade_{m}^1 (-\mathcal{G}, \sample, -\vsigma)
    + 3 \frange \sqrt{\frac{\ln \frac{2}{\eta}}{2m}} \enspace.
  \]
  Following the same steps as for $\mathsf{E}_1$, we have that $\mathsf{E}_2$
  holds with probability at least $1 - \tfrac{\eta}{2}$, as the fact that we are
  considering $\erade_{m}^1 (-\mathcal{G}, \sample, -\vsigma)$ rather than
  $\erade_{m}^1 (-\mathcal{G}, \sample, \vsigma)$ is not influential.

  It is easy to see that  $\erade_{m}^1 (-\mathcal{G}, \sample, -\vsigma) =
  \erade_{m}^1 (\mathcal{G}, \sample, \vsigma)$, and that
  \[
    \sup_{g \in -\mathcal{G}} \left( \EE_\sample[g] - \E_\probdist[g] \right) =
    \sup_{g \in \mathcal{G}} \left( \E_\probdist[g] - \EE_\sample[g] \right)
    \enspace.
  \]
  Thus we can rewrite $\mathsf{E}_2$ as
  \[
    \mathsf{E}_2 = \sup_{g \in \mathcal{G}} \left( \E_\probdist[g] -
    \EE_\sample[g] \right) \le 2 \erade_{m}^1 (\mathcal{G}, \sample, \vsigma)
    + 2 \frange \sqrt{\frac{\ln \frac{2}{\eta}}{2m}} \enspace.
  \]
  From the union bound, we have that $\mathsf{E}_1$ and $\mathsf{E}_2$ hold
  simultaneously with probability at least $1- \eta$, i.e., the following event
  holds with probability at least $1-\eta$
  \[
    \sd(\mathcal{G}, \sample, \probdist) \le 2 \erade_{m}^1 (\mathcal{G},
    \sample, \vsigma) + 3 \frange \sqrt{\frac{\ln \frac{2}{\eta}}{2m}} \enspace.
  \]
  The thesis then follows from the fact $\supdev = \sd(\mathcal{G}, \sample,
  \probdist)$.
\end{proof}

\else 
Additional missing proofs are in the appendix of the extended online version at
\refmissing.
\fi 

\begin{proof}[Proof of \cref{thm:correctness}]
  For ease of notation, let $\mathcal{G} = \negborder(\tfp(\theta, \lang))$.
  Let $\rho$, $r$, and $\varepsilon$ be as in \cref{thm:supdevvar} for $\eta =
  \delta$ and $\F = \mathcal{G}$. \cref{thm:supdevvar} tells us that, with
  probability at least $1 - \delta$, it holds $\sd(\mathcal{G}, \sample) \le
  \varepsilon$.\footnote{We actually only need a value $\varepsilon$ such that
    $\sup_{f \in \mathcal{G}} \left( \EE_\sample[f] - \E_\probdist[f] \right) <
  \varepsilon$, but the gain would be minimal and it would make the
  presentation more complicated.} Assume from now on that that is the case.

  We use this fact to show inductively that, at the end of every iteration of
  the loop of \algotfp\ (lines~\ref{algline:startloop}--\ref{algline:endloop} of
  \cref{algo:tfp}), it  holds that $\mathcal{G} \subseteq \tempset$ and $Y
  \subseteq \tfp(\theta, \lang)$, and therefore the thesis will hold.

  Consider the first iteration of the loop. We have $\tempset = \lang \supseteq
  \mathcal{G}$. Let $\hat{\rho}$, $\hat{r}$, and $\hat{\varepsilon}$ be the
  values computed inside the call to the function \GetSupDevBoundVar\ on
  line~\ref{algline:eps} with the parameters mentioned in the description of the
  algorithm. It holds that $\hat{\rho} \ge \rho$, because the $n$-MCERA of a
  superset of a family is not smaller than the $n$-MCERA of the family. It
  follows that $\hat{r} \ge r$, which in turn implies that $\hat{\varepsilon}
  \ge \varepsilon$. Since we assumed that $\sd(\mathcal{G}, \sample) \le
  \varepsilon$, we have $\hat{\varepsilon} \ge \varepsilon \ge \sd(\mathcal{G},
  \sample)$
  No function $f \in \mathcal{G}$ may then have sample mean
  $\EE_\sample[f]$ greater than or equal to $\theta + \hat{\varepsilon}$, as
  every such $f$ has $\E_\probdist[f] < \theta$. Call this fact $\mathsf{A}$. A
  first consequence of $\mathsf{A}$ is that, at the end of the iteration, it
  holds $\mathcal{G} \subseteq \tempset$. A second consequence of $\mathsf{A}$
  and of the antimonotonicity property is that \emph{none} of the functions $f
  \in \lang$ such that $\E_\probdist[f] < \theta$ may have $\EE_\sample[f] \ge
  \theta + \hat{\varepsilon}$. Equivalently, only functions $f \in \lang$ such
  that $\E_\probdist[f] \ge \theta$, i.e., such that $f \in \tfp(\theta,
  \lang)$, may have $\EE_\sample[f] \ge \theta + \hat{\varepsilon}$, i.e.,
  $\tempset' \setminus \tempset \subseteq \tfp(\theta, \lang)$, so $Y \subseteq
  \tfp(\theta, \lang)$ at the end of the first iteration. The base case is
  complete.

  Assume now that $\mathcal{G} \subseteq \tempset$ and $Y \subseteq \tfp(\theta,
  \lang)$ at the end of all iterations from $1$ to $i$. Following the same
  reasoning as for the base case, it holds that these facts are true also at the
  end of iteration $i+1$ and our proof is complete.
\end{proof}

\begin{table}[tbh]
   \centering
\iflongversion%
\small{%
\else
\scriptsize{%
\fi
   \begin{tabular}{lrrccc}
      \toprule
      dataset & $|\D|$ & $|I|$ & avg.~trans.~len. \\
      \midrule
      svmguide3 & 1,243 & 44 & 21.9 \\
      chess  & 3,196 & 75 & 37 \\
      breast cancer  & 7,325 & 396 & 11.7 \\
      mushroom  & 8,124 & 117 & 22   \\
      phishing & 11,055 & 137 & 30 \\
      a9a & 32,561 & 245 & 13.9 \\
      pumb-star& 49,046 & 7,117 & 50.9 \\
      bms-web1 & 58,136 & 60,878 & 3.51 \\
      connect& 67,557 & 129 & 43.5 \\
      bms-web2& 77,158 & 330,285 & 5.6 \\
      retail & 87,979 & 16,470 & 10.8 \\
      ijcnn1  & 91,701 & 43 & 13 \\
      T10I4D100K & 100,000 & 1,000 & 10 \\
      T40I10D100K& 100,000 & 1,000 & 40 \\
      accidents & 340,183 & 468 & 34.9 \\
      bms-pos& 515,420 & 1,657 & 6.9 \\
      covtype & 581,012 & 108 & 12.9 \\
      susy & 5,000,000 & 190 & 19 \\
      \bottomrule
   \end{tabular}
  }
\caption{Datasets statistics. For each dataset, we report the number $|\D|$ of
transactions; the number $| \Itms |$ of items; the average transaction
length.}\label{tab:data}
\Description[Table 1: Datasets statistics.]{Table 1: Datasets statistics.}
\end{table}

\begin{figure}[tbh]
  \includegraphics[width=0.35\textwidth]{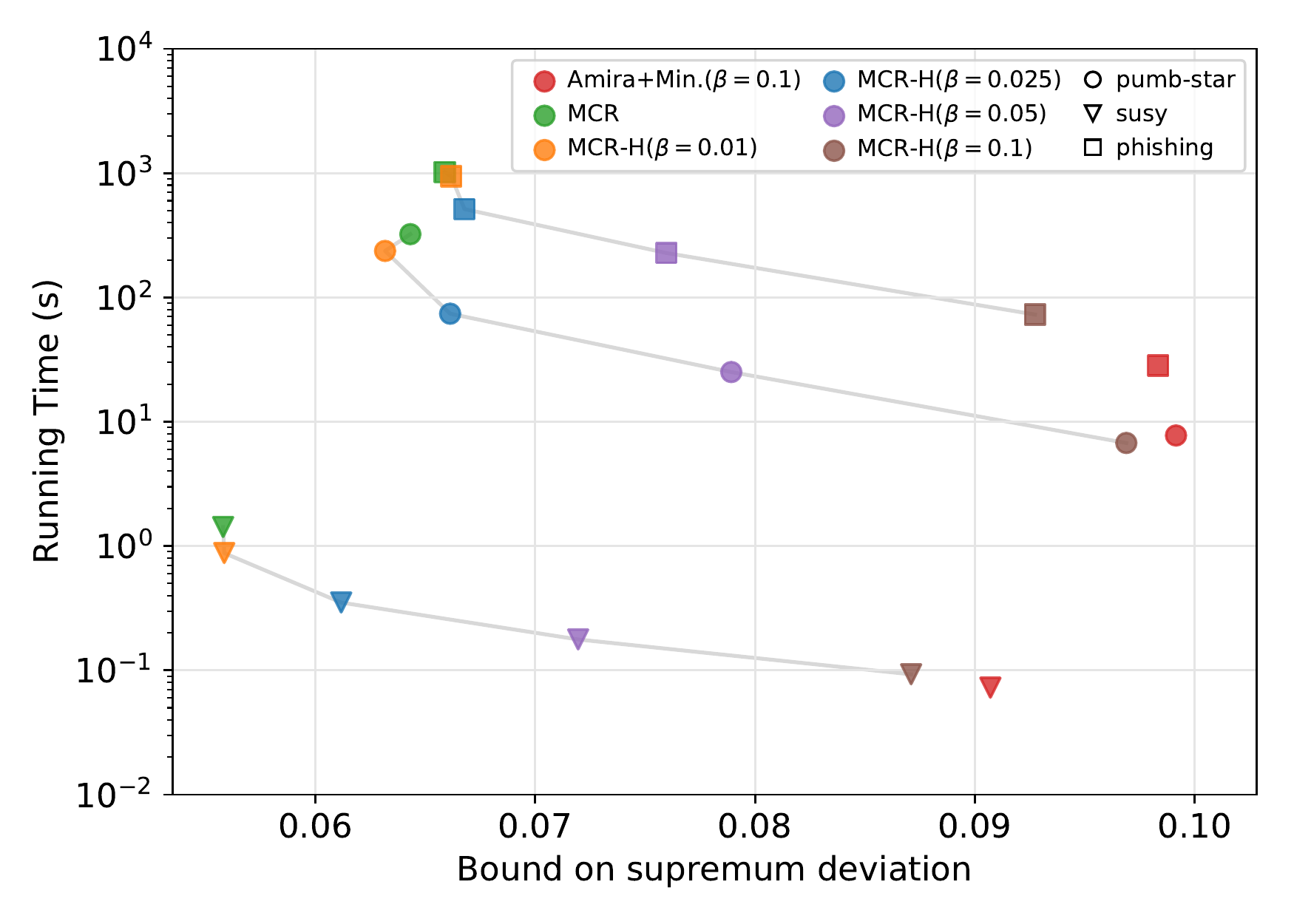}
  \caption{Running times of \algo, \algoh\ and \amira\ vs corresponding upper
  bound on supremum deviation of the entire set of functions $\F$. For \algoh\
  we use different values of $\beta$. $y$-axis in log scale but $x$ axis is
  linear. Each marker shape corresponds to one of the
datasets.}\label{fig:runningtimes_pareto}
\Description[Figure 3. MCRapper-H allows to trade-off between the bound on the SD and running time.]{Figure 3: The running time of MCRapper can be reduced using MCRapper-H, at the cost of computing an higher bound on the supremum deviation.}
\end{figure}

\vspace{-15pt}
\subsection{Reproducibility}\label{sec:repro}
We now describe how to reproduce our experimental results. Code and data are
available at \repolink.

The code of \algo, \algotfp, and \amira\ are in the sub-folders
\texttt{mcrapper/} and \texttt{amira/}. To compile with recent GCC or Clang, use
the \texttt{make} command inside each sub-folder.

The convenient scripts \texttt{run\_amira.py} and \texttt{run\_mcrapper.py} can
be used to run the experiments (i.e., run \amira, \algo, and
\algotfp). They accept many input parameters (described using the flag
\texttt{-h}). You need to specify a dataset and the size of a random sample to
create using the flags \texttt{-db} and \texttt{-sz}. E.g., to process a random
sample of $10^3$ transactions from the dataset \texttt{mushroom} with $n=100$,
run
\begin{center}
  \begin{verbatim}
    run_mcrapper.py -db mushroom -sz 1000 -j 100
  \end{verbatim}
\end{center}
\vspace{-12pt}
and it automatically executes both \amira\ and \algo. The command line to
process with \algotfp\ a sample of $10^4$ transactions from the
dataset \texttt{retail} with $n=10$ and $\theta = 0.05$ is
\begin{center}
  \begin{verbatim}
    run_mcrapper.py -db retail -sz 10000 -j 10 -tfp 0.05
  \end{verbatim}
\end{center}
\vspace{-12pt}
The \texttt{run\_all\_datasets.py} script runs all the instances of \algo\ and
\amira\ in parallel, and can be used to reproduce all the experiments described
in \cref{sec:experiments}. The \texttt{run\_tfp\_all\_datasets.py} script
reproduces the experiments for \algotfp\ and \algotfpamira.

All the results are stored in the files
\texttt{results\_mcrapper.csv} and \texttt{results\_tfp\_mcrapper.csv}.


\iflongversion%
\fi

\end{document}